\documentclass[10pt,journal,letterpaper,twoside]{IEEEtran}

\usepackage[T1]{fontenc}
\usepackage{newtxtext}
\usepackage{amsmath,amssymb,mathtools}
\usepackage{amsthm}
\usepackage{newtxmath}
\usepackage{booktabs, multirow, array}
\usepackage[table]{xcolor}
\usepackage{colortbl}
\usepackage{algorithm}
\usepackage[noend]{algpseudocode}
\usepackage{enumitem}
\usepackage{xspace}
\usepackage{microtype}
\usepackage{graphicx}
\usepackage{pifont}
\usepackage{listings}
\usepackage{url}
\usepackage{cite}
\usepackage{placeins}
\usepackage[most]{tcolorbox}
\usepackage[hidelinks]{hyperref}

\graphicspath{{./}{fig/}{logos/}}
\makeatletter
\def\input@path{{./}}
\makeatother

\definecolor{Ink}{HTML}{20262E}
\definecolor{Navy}{HTML}{274C77}
\definecolor{Teal}{HTML}{2A8178}
\definecolor{Orange}{HTML}{C66B2B}
\definecolor{Brick}{HTML}{A64B3C}
\definecolor{Line}{HTML}{C8D2D9}
\definecolor{PaleBlue}{HTML}{F0F5F8}
\definecolor{PaleTeal}{HTML}{F1F8F5}
\definecolor{PaleGray}{HTML}{F5F6F7}
\definecolor{PaleOrange}{HTML}{FFF5ED}
\definecolor{SkillZipPaleGrey}{HTML}{ECEEF0}
\definecolor{LossyGrey}{HTML}{7A848C}
\definecolor{DupYellow}{HTML}{FCEFC4}   
\definecolor{KeepGreen}{HTML}{D4E8DC}   
\definecolor{GuardBlue}{HTML}{D6E6F5}   
\definecolor{PreprintBlue}{HTML}{28658A}
\definecolor{PreprintBlueDark}{HTML}{1F4F6C}
\definecolor{PreprintGreen}{HTML}{2C805B}
\definecolor{PreprintGreenDark}{HTML}{226548}
\definecolor{PreprintOrange}{HTML}{C5672D}
\definecolor{PreprintRule}{HTML}{626A70}
\definecolor{SkillZipBlue}{RGB}{35,79,119}
\definecolor{SkillZipGreen}{RGB}{45,105,78}
\definecolor{SkillZipPaleBlue}{RGB}{232,240,247}
\definecolor{SkillZipPaleGray}{RGB}{245,246,247}

\newtcolorbox{abstractbox}{
  enhanced,
  breakable,
  colback=PaleBlue,
  colframe=PreprintBlue,
  boxrule=0.55pt,
  arc=2pt,
  left=5pt,
  right=5pt,
  top=4pt,
  bottom=4pt,
  before skip=0pt,
  after skip=5pt,
  fontupper=\small
}

\newtcolorbox{projectbox}{
  enhanced,
  breakable,
  colback=PaleTeal,
  colframe=PreprintGreen,
  boxrule=0.55pt,
  arc=2pt,
  left=5pt,
  right=5pt,
  top=3pt,
  bottom=3pt,
  before skip=0pt,
  after skip=7pt,
  fontupper=\small
}

\newtcolorbox{glancebox}{
  enhanced,
  breakable,
  title={SkillZip Pro at a glance},
  colback=PaleOrange,
  colframe=PreprintOrange,
  colbacktitle=PreprintOrange,
  coltitle=white,
  fonttitle=\bfseries\sffamily\footnotesize,
  fontupper=\small,
  boxrule=0.55pt,
  arc=2pt,
  left=5pt,
  right=5pt,
  top=3pt,
  bottom=3pt,
  toptitle=2pt,
  bottomtitle=2pt,
  before skip=5pt,
  after skip=7pt
}

\newtcolorbox{takeawaybox}{
  enhanced,
  title={Takeaway},
  colback=PaleTeal,
  colframe=PreprintGreen,
  colbacktitle=PreprintGreen,
  coltitle=white,
  fonttitle=\bfseries\sffamily\footnotesize,
  fontupper=\small,
  boxrule=0.55pt,
  arc=2pt,
  left=5pt,
  right=5pt,
  top=3pt,
  bottom=3pt,
  toptitle=2pt,
  bottomtitle=2pt,
  before skip=5pt,
  after skip=6pt
}

\newtcolorbox{insightbox}[1]{
  enhanced,
  breakable,
  colback=PaleOrange,
  frame hidden,
  borderline west={2.6pt}{0pt}{PreprintOrange},
  sharp corners,
  boxrule=0pt,
  fontupper=\small,
  coltitle=PreprintOrange,
  fonttitle=\bfseries\sffamily\footnotesize,
  detach title,
  before upper={\tcbtitle\ \ },
  title={\textsc{Insight #1}\,\raisebox{-0.5pt}{\ding{72}}},
  left=9pt,
  right=6pt,
  top=3pt,
  bottom=3pt,
  before skip=6pt,
  after skip=6pt
}

\newtcolorbox{roadmapbox}{
  enhanced,
  breakable,
  title={Appendix roadmap},
  colback=PaleBlue,
  colframe=PreprintBlue,
  colbacktitle=PreprintBlue,
  coltitle=white,
  fonttitle=\bfseries\sffamily\small,
  fontupper=\small,
  boxrule=0.55pt,
  arc=2pt,
  left=6pt,
  right=6pt,
  top=4pt,
  bottom=4pt,
  toptitle=2pt,
  bottomtitle=2pt,
  before skip=4pt,
  after skip=8pt
}

\newtcolorbox{promptbox}[1][]{
  enhanced,
  breakable,
  title={#1},
  colback=PaleBlue,
  colframe=PreprintBlue,
  colbacktitle=PreprintBlue,
  coltitle=white,
  fonttitle=\bfseries\sffamily\footnotesize,
  fontupper=\ttfamily\footnotesize,
  boxrule=0.5pt,
  arc=2pt,
  left=6pt,
  right=6pt,
  top=4pt,
  bottom=4pt,
  toptitle=2pt,
  bottomtitle=2pt,
  before skip=5pt,
  after skip=6pt,
  before upper={\setlength{\parskip}{2pt}}
}

\newcommand{\method}{\textsc{SkillZip Pro}\xspace}
\newcommand{\base}{\textsc{SkillZip}\xspace}
\newcommand{\E}{\mathbb{E}}

\newcommand{\contract}{\mathcal{C}}

\newcommand{\cmark}{\ding{51}}
\newcommand{\xmark}{\ding{55}}
\newcommand{\Description}[1]{}
\newcommand{\hl}[1]{\textcolor{SkillZipBlue}{\textbf{#1}}}
\newcommand{\win}[1]{\textcolor{SkillZipGreen}{\textbf{#1}}}
\newcommand{\caveat}[1]{\textcolor{LossyGrey}{\textbf{#1}}}
\newcommand{\dup}[1]{\colorbox{DupYellow}{#1}}       
\newcommand{\kept}[1]{\colorbox{KeepGreen}{#1}}       
\newcommand{\guard}[1]{\colorbox{GuardBlue}{#1}}      
\newcommand{\gone}[1]{\textcolor{LossyGrey}{#1}}      

\theoremstyle{definition}
\newtheorem{definition}{Definition}[section]
\newtheorem{assumption}[definition]{Assumption}
\theoremstyle{plain}
\newtheorem{proposition}[definition]{Proposition}

\lstdefinestyle{compactjson}{
  basicstyle=\ttfamily\scriptsize,
  columns=fullflexible,
  breaklines=true,
  frame=single,
  rulecolor=\color{Line},
  backgroundcolor=\color{PaleGray},
  xleftmargin=2pt,
  xrightmargin=2pt,
  aboveskip=5pt,
  belowskip=5pt,
  showstringspaces=false
}

\setlist[itemize]{leftmargin=*,topsep=2pt,itemsep=1pt,parsep=0pt,partopsep=0pt}
\setlist[enumerate]{leftmargin=*,topsep=2pt,itemsep=1pt,parsep=0pt,partopsep=0pt}
\AtBeginDocument{\setlength{\parfillskip}{0pt plus 0.76\columnwidth}}

\newcommand{\AffiliationLogoStrip}{%
  \parbox{\textwidth}{%
    \normalfont\normalsize
    \raggedright
    \raisebox{0.010in}[0.33in][0pt]{%
      \includegraphics[height=0.255in,keepaspectratio]{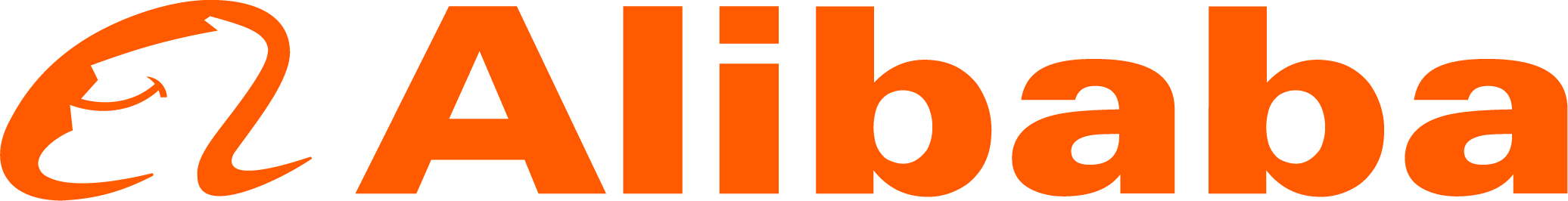}}%
    \hspace{0.30in}%
    \raisebox{-0.055in}[0.5in][0pt]{%
      \includegraphics[height=0.4in,keepaspectratio]{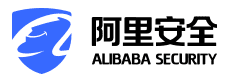}}%
    \hspace{0.30in}%
    \raisebox{0.010in}[0.33in][0pt]{%
      \includegraphics[height=0.315in,keepaspectratio]{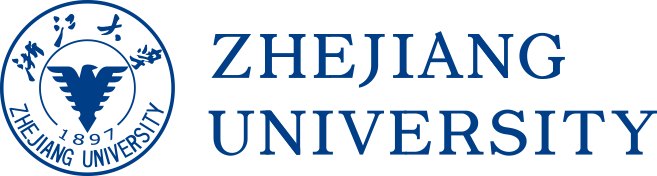}}%
    \par\vspace{0.05in}%
    \noindent\color{PreprintRule}\rule{\textwidth}{0.55pt}%
    \par\vspace{0.235in}%
  }%
}

\hypersetup{
  pdftitle={SkillZip Pro: Execution-Aware Dynamic Compression of Progressively Loaded Skills for Self-Evolving Agents},
  pdfauthor={Xiaofan Bai, Chao Liu, Hongqiang Lin,  Xuan Jin, Xipeng Cao, Yuhong Li},
  pdfsubject={Harness-free, bundle-aware compression of progressively loaded agent skills},
  pdfkeywords={LLM agents, agent skills, progressive disclosure, skill compression, minimum description length}
}

\title{%
  \AffiliationLogoStrip\par
  {\normalfont\fontsize{20.5}{24.0}\selectfont
  SkillZip Pro: Execution-Aware Dynamic Compression of\\[-0.12em]
  Progressively Loaded Skills for Self-Evolving Agents}
}

\author{%
  {\normalsize Xiaofan Bai$^{1}$, Chao Liu$^{1}$, Hongqiang Lin$^{2}$, Di Wu$^{1}$,  Mingli Song$^{2}$, 
  Xuan Jin$^{1,\dagger}$, Xipeng Cao$^{1}$, Yuhong Li$^{1}$}\\[-0.12em]
  {\footnotesize $^{1}$Alibaba Group \quad $^{2}$Zhejiang University 
  }\\[-0.14em]
  {\scriptsize\ttfamily
  baixiaofan.bxf@alibaba-inc.com
  }\\[-0.18em]
  
  {
  $^{\dagger}$Project leader}%
}
\IEEEtitleabstractindextext{%
\begin{abstractbox}
\textbf{Abstract---}
Production agent skills are directory bundles, not isolated prompts. The root is loaded at activation; references, schemas, scripts, assets, and nested subskills are loaded only when an execution path needs them. Compressing only the root misses most deployment cost and may move branch-specific details into the always-loaded context. Flattening instead destroys progressive-loading boundaries.

We introduce \method, an evaluation-free compressor for complete, progressively loaded skill bundles. It leaves the agent harness unchanged and emits an ordinary directory. The method combines two safeguards. First, it compresses \emph{across files}, removing content from a reference or subskill when the root or a declared environment contract already provides it. Second, it preserves routing, so every required file and directly callable entry remains reachable after rewriting. Users can configure \method along two independent axes. \emph{One-Shot} mode rebuilds the full bundle; \emph{Continual} mode reuses state and applies Zip-on-Write after each evolution patch. \emph{Persistent} compression rewrites the shipped bundle to reduce storage and runtime context. \emph{Transient} compression keeps that bundle byte-identical and builds a task-specific view, reducing only per-run context after build cost. Entry contracts mark private, public, and conditional resources; a multi-entry audit preserves standalone public subskills.

On a production content-moderation skill evaluated by our industrial multi-round harness, \method removes \hl{38\%} of skill bundle tokens and \hl{10.4\%} of end-to-end per-run tokens with no quality loss, while an unprotected 71\% configuration loses up to 26 accuracy points to one-sided false positives. On a multi-entry bundle, \method effeciently reduces token cost while near-perfectly preserving every route and public entry.

\end{abstractbox}
\begin{projectbox}
\textbf{Available at:}
\url{https://github.com/yutou520131/SkillZip-Pro}
\end{projectbox}
}

\begin{document}
\maketitle
\thispagestyle{empty}
\IEEEdisplaynontitleabstractindextext

\section{Introduction}
\label{sec:intro}

Agent skills package reusable instructions, procedures, and tools for use across tasks. A prototype may represent a skill as one prompt, but a production skill is usually a directory. A short root declares when the skill applies; references and subskills hold branch-specific knowledge; scripts implement deterministic operations; and schemas constrain inputs and outputs. Agents load this bundle progressively: catalog metadata is visible before selection, the root appears after activation, and auxiliary files are opened only when the current path requires them~\cite{anthropic2025claudecode,openai2025codex}.

This execution mode changes the compression objective. A root-only method can report an attractive ratio while leaving most deployed text untouched; it can even move rare branch details into the root and increase every invocation's context. Concatenation fixes the accounting gap but destroys progressive disclosure. Neither reflects the agent's actual cost.

\begin{figure}[t]
  \centering
  \includegraphics[width=1.\columnwidth]{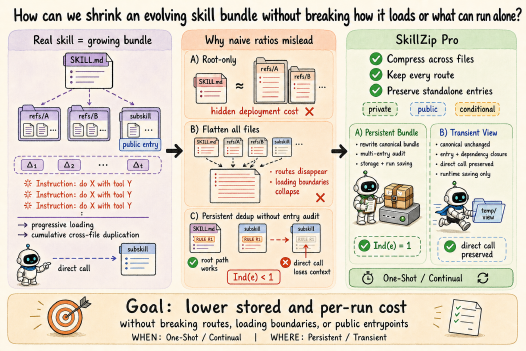}
  \caption{The dynamic compression and progressive loading of agent skills.}
  \Description{A line chart showing skill tokens continuing to grow while unique procedural content rises early and then saturates.}
  \label{fig:growth}
\end{figure}

Self-evolving agents make the problem larger. Systems such as Voyager, Reflexion, and later evolution frameworks accumulate successful routines, counterexamples, and repair rules over time~\cite{wang2023voyager,shinn2023reflexion,zhang2025ace,gao2025selfevolving}. Repetition grows both within files and across branches. Sharing that logic saves space, but placement matters: moving content used by two rare branches into the root charges every task. Compression must therefore optimize a \emph{progressively loaded resource graph}, not a flat string.

Existing methods are not suitable for this setting. Prompt compressors optimize token salience or reconstruction in a flat context~\cite{jiang2023llmlingua,jiang2023longllmlingua,pan2024llmlingua2}. SkillReducer shortens skills with evaluation feedback, which requires rollouts and ties compression to sampled tasks~\cite{gao2026skillreducer}. Our earlier \base formulation instead found a shortest faithful representation without evaluations~\cite{bai2026skillzipevaluationfreeskillcompression}. It preserved typed behavioral contracts and supported Zip-on-Write, but treated one document at a time and could not model references, subskills, or progressive-loading costs.

We present \method, a bundle-aware extension of \base that preserves a harness-agnostic constraint: \emph{the agent harness is unchanged}. Its output is an ordinary skill directory that uses existing file readers, paths, and loading behavior; no resolver, interception hook, or runtime protocol is required.

\method adds two capabilities that a single-document compressor lacks. \textbf{\textcolor{SkillZipBlue}{Pillar~1, activation-aware cross-file compression,}} removes content already supplied by the root or a declared environment contract, factors repeated branch content within its loading scope, and moves long guarded branches into on-demand capsules. These transformations reduce cross-file redundancy without enlarging the always-loaded layer. \textbf{\textcolor{PreprintGreen}{Pillar~2, cross-file routing preservation,}} locks routing instructions and audits the materialized directory before publication. Every required file must remain reachable, and interface contracts such as schemas, label lists, and worked formats remain contiguous and verbatim.
The compiler supports two compression modes. \emph{One-Shot Compression} scans and optimizes the full directory for migration, release, or periodic repacking. \emph{Continual Compression} applies each evolution patch verbatim, reuses unchanged contracts, and repairs only the affected graph closure, where a global repack is used to bound the accumulated drift. 

To accommodate whether referenced files and subskills must remain independently usable outside the root skill, SkillZip Pro introduces a separate output-lifecycle choice. \textbf{Persistent Bundle Compression} rewrites the canonical directory and is best suited to private resources used only through the root, reducing both storage and future runtime context. If it modifies a public entry, a multi-entry audit must verify that the entry remains independently usable. \textbf{Transient Execution-View Compression} instead leaves the canonical directory unchanged and constructs a task-specific view for each run, preserving the independent usability of all shipped entries by construction while reducing runtime context but not disk usage. Either lifecycle can use either compression mode. Section~\ref{sec:lifecycle} defines the resulting four combinations and their costs.


Both schedules resolve explicit local references into a conservative resource graph, extract typed contracts from text, and lock executable or binary artifacts. The objective separates catalog, activation, deployment, and path-weighted costs. After transformation, an independent audit rereads the emitted directory and checks graph closure, scope, contract coverage, and byte identity before atomic publication.


Our contributions are:
\begin{itemize}
\item \textbf{\textcolor{SkillZipBlue}{Bundle-aware compression for progressively loaded skills.}} We extend skill compression from a single document to a resource graph and jointly optimize content placement across the root, reference files, and subskills. This design removes cross-file redundancy without moving rarely used content into the always-loaded root.

\item \textbf{\textcolor{PreprintGreen}{Safe routing and independent entry preservation.}} We treat routing instructions and entry contracts as explicit constraints and audit the materialized bundle before deployment. This ensures that required branches remain reachable and that public subskills and references remain independently usable.

\item \textbf{Flexible compression for different update and deployment requirements.} We separate when compression is performed from where its output is stored. One-Shot and Continual modes support static and self-evolving skills, while Persistent and Transient lifecycles accommodate different requirements for storage reduction and independent resource usability.

\item \textbf{Comprehensive compression results with production evidence.} Every removal is supported by a containment, coverage, or logged-entailment witness, while interface contracts remain intact. On a production skill, \method reduces deployed content by \hl{38\%} and end-to-end per-run tokens by \hl{10.4\%} while preserving decision quality.
\end{itemize}

\section{Related Work}
\label{sec:related}

\textbf{\textcolor{SkillZipBlue}{Self-evolving agents and skill memories.}}
Agents increasingly retain reusable procedures instead of solving every task from scratch. Voyager stores executable skills, Reflexion records verbal feedback, and ACE organizes evolving context as structured playbooks~\cite{wang2023voyager,shinn2023reflexion,zhang2025ace}. Other systems revise, optimize, or formalize skill artifacts~\cite{liu2026skillrevise,wang2026skillgrad,yang2026skillopt,zhang2026formalskill}. As these libraries grow, equivalent rules accumulate at different scopes and branch updates can conflict. \method compresses the resulting representation regardless of how it was created.

\textbf{\textcolor{SkillZipBlue}{Prompt and context compression.}}
LLMLingua, LongLLMLingua, and LLMLingua-2 remove or rewrite low-utility tokens~\cite{jiang2023llmlingua,jiang2023longllmlingua,pan2024llmlingua2}; selective-context, gisting, and recompression methods also target flat inference contexts~\cite{li2023selective,mu2023gisting,xu2024recomp}. Skill bundles differ in two ways: rare content may still be mandatory, and each execution loads only part of the directory. \method therefore compresses typed behavioral units while preserving loading boundaries.

\textbf{\textcolor{SkillZipBlue}{Skill compression and runtime representations.}}
SkillReducer uses evaluation feedback to compress skill text~\cite{gao2026skillreducer}; parameterized skills and execution-time mechanisms instead change the representation or runtime~\cite{zhang2026skilltolora,chen2026skillrt}. Our transient lifecycle also builds content per run, but it remains evaluation-free and harness-neutral: the agent receives an ordinary directory and no resolver observes execution. The closest precursor is \base, which introduced typed coverage and shortest-cover compression for one document. \method extends its optimization object, cost model, rewrites, and audit to progressively loaded bundles.

\textbf{\textcolor{SkillZipBlue}{Grammar-based compression and MDL.}}
Minimum description length (MDL) selects the shortest faithful representation~\cite{grunwald2007mdl,galbrun2022mdl}. Grammar compressors such as Sequitur and Re-Pair factor sequences when references amortize definition cost~\cite{nevill1997sequitur,larsson2000repair}. We add semantic types, activation scope, path weights, locked artifacts, and deployment constraints. Sharing is allowed only when it lowers loaded-path cost.

\section{Problem Formulation}
\label{sec:problem}

\subsection{A Skill Is a Progressively Loaded Bundle}

We model a skill as a rooted directory $B=(V,E,r)$. Each resource $v\in V$ has a canonical path $p_v$, bytes $b_v$, media type $m_v$, and loading class $\ell_v$. The root $r$ is normally \texttt{SKILL.md}. An edge $e=(u,v,g,s)$ records a reference from $u$ to $v$, its guard $g$, and source span $s$; a guard may be unconditional or author-written, such as ``for CSV export.''

We distinguish four loading layers:
\begin{enumerate}
  \item \textbf{Catalog:} the name, description, and entry metadata visible before activation;
  \item \textbf{Activation:} the root instructions loaded whenever the skill is selected;
  \item \textbf{Path:} the transitive resources loaded for a particular execution branch;
  \item \textbf{Deployment:} all bytes distributed with the skill, including resources rarely or never read.
\end{enumerate}
An edit can improve one layer while harming another. Moving a rare 500-token branch into the root may reduce \hl{deployment} length after deduplication yet add 500 tokens to every activation. We therefore measure each layer separately.

\begin{definition}[Safe resource graph]
A graph is \emph{safe} if every resolved target is a regular file inside the canonical bundle root, every recorded internal reference has an existing target, and every edge retains its source span and guard. External URLs are recorded as external edges but are never fetched. Symbolic-link escape, path traversal, missing targets, and ambiguous dynamic paths are rejected in strict mode.
\end{definition}

The resolver is intentionally conservative. Markdown links, explicit path literals, and declared subskill entries are recognized; an unrecognized file remains in the deployment bundle but contributes no inferred dependency. This prevents the compressor from inventing loading semantics.

\subsection{Entry Contracts: Which Resources Must Stand Alone}
\label{sec:entrycontracts}

Agents usually reach resources from the root, but they may also select a nested subskill or reference directly. The compressor must know this before deleting content: text covered by the root may still be essential to a direct call. An \emph{entry contract} $\eta_v$ therefore assigns each node one of three roles:
\begin{itemize}
  \item \textbf{Private (internal):} reached only through the root, with no standalone-use requirement.
  \item \textbf{Public (standalone):} directly selectable and therefore fully usable without the root.
  \item \textbf{Conditional:} directly usable only with a declared host context or dependency closure.
\end{itemize}
The author or catalog declares the contract; the compressor never infers it. For a public or conditional entry $e$, define independence as
\begin{equation}
\operatorname{Ind}(e,B')=
\operatorname{Cov}(e,B'_{\downarrow e})
\,\mathbb{I}[e\text{ discoverable}],
\label{eq:independence}
\end{equation}
where $\operatorname{Cov}$ is the fraction of source contract units covered by the closure $B'_{\downarrow e}$ reachable from $e$, including any declared host context. Discoverability requires an executable entry at the advertised path. Thus $\operatorname{Ind}=1$ means that the entry remains independently usable; missing content lowers coverage, while a renamed or hidden entry has independence zero.

\subsection{Typed Resource Contracts}

For each textual node $v$, \method extracts a contract
\begin{equation}
\contract_v=\langle I_v,W_v,T_v,R_v,O_v,E_v,P_v,L_v\rangle,
\label{eq:contract}
\end{equation}
where $I$ contains applicability and interface conditions; $W$ workflow states and ordering edges; $T$ tool or resource requirements; $R$ rules and prohibitions; $O$ output fields and formats; $E$ evidence or verification obligations; $P$ provenance to source spans; and $L$ locked residuals that must remain verbatim. Each unit $a\in\contract_v$ has a semantic type, normalized payload, scope, guard, and provenance.

Executable code, structured data, images, and other non-instructional artifacts are \emph{locked nodes}. Phase A may rename neither their path nor their bytes. Their inbound references can be rewritten only when the referring Markdown is rewritten and the canonical target remains identical.

For a candidate bundle $B'$, let $\operatorname{cover}(a,B')$ mean that a compatible statement, reference, or locked artifact in $B'$ entails unit $a$ on every path where it originally applied. Coverage is type constrained: an example cannot cover a prohibition, generic advice cannot cover a required output field, and an unrelated branch cannot cover a guarded rule.

\begin{definition}[Bundle faithfulness]
$B'$ is faithful to $B$ under environment contract $H$ if (i) every source unit is covered at a compatible scope, (ii) every locked node is byte-identical, (iii) all internal references in $B'$ resolve safely, and (iv) any unit removed by host entailment has an exact typed witness in $H$ bound to the audited environment digest.
\end{definition}

This structural guarantee does not assert behavioral equivalence for every language model. It protects the bundle's explicit contract, not every latent reading.

\subsection{Four Costs, One Constrained Objective}

Let $\tau(x)$ be the deployment tokenizer or another declared length function. For a bundle $B'$, define
\begin{align}
C_{\mathrm{cat}}(B') &= \tau(\text{name, description, entry}), \\
C_{\mathrm{act}}(B') &= \tau(b'_r), \\
C_{\mathrm{dep}}(B') &= \sum_{v\in V'} \tau(b'_v), \\
C_{\mathrm{path}}(B';\pi) &= \sum_{v\in\operatorname{load}(B',\pi)} \tau(b'_v).
\end{align}
Here $\pi$ denotes an execution path or task class and $\operatorname{load}$ follows the unchanged agent's progressive-loading behavior. Given an empirical path distribution $q(\pi)$, our primary optimization target is
\begin{equation}
J(B')=C_{\mathrm{cat}}(B')+C_{\mathrm{act}}(B')
+\E_{\pi\sim q}C_{\mathrm{path}}(B';\pi)
+\lambda C_{\mathrm{dep}}(B'),
\label{eq:objective}
\end{equation}
subject to bundle faithfulness. We use $\lambda=0.05$ as a transparent default so that storage matters without dominating runtime exposure. If no trace distribution is available, explicit branch guards induce a uniform distribution over reachable leaf paths; every reported result must identify which estimator was used.

We report all four costs even when optimizing Eq.~\eqref{eq:objective}. A single ``compression ratio'' is insufficient: it can conceal a regression in the always-loaded root or in the tail of \hl{path} cost. For metric $x$, the reduction is $1-C_x(B')/C_x(B)$; negative values are preserved rather than clipped.

\subsection{Two Deployment Lifecycles and Their Costs}
\label{sec:lifecyclecosts}

The four costs above describe one artifact, but deployment has two distinct lifecycles. \textbf{\hl{Persistent}} compression rewrites the canonical bundle; its one-time work reduces storage and future runs. \textbf{\hl{Transient}} compression leaves the canonical bundle unchanged and builds a task-specific execution view $\widehat{B}_{e,\pi}$ before a run through entry $e$. Its benefit applies only to that run and must be reported after build and cache cost.

Because the two change different things, we separate seven quantities and never average across them:
\begin{enumerate}
  \item \textbf{Canonical storage} $C_{\mathrm{disk}}(B')$: bytes of the on-disk bundle. Persistent shrinks it; transient keeps it equal to the source.
  \item \textbf{Shipped/deployment} $C_{\mathrm{dep}}(B')$ (Eq.~\eqref{eq:objective}): distributed bytes.
  \item \textbf{Per-entrypoint activation} $C_{\mathrm{act}}(B';e)$: entry tokens loaded when entry $e$ is selected.
  \item \textbf{Per-run execution-view tokens} $C_{\mathrm{view}}(\widehat{B}_{e,\pi})$: context loaded for one run through $e$ on task $\pi$.
  \item \textbf{Transient build latency} $\Lambda_{\mathrm{build}}(e,\pi)$: wall-clock time and any model calls to construct $\widehat{B}_{e,\pi}$.
  \item \textbf{Cache and invalidation} $C_{\mathrm{cache}}$: storing views keyed by (bundle digest, entry, environment version) and rebuilding them when any key changes.
  \item \textbf{Public-entrypoint independence} $\operatorname{Ind}(e,B')=1$ for every public $e$: a mandatory deployment constraint.
\end{enumerate}
Persistent results report quantities 1--3 and per-run load. Transient results report quantities 4--6, including build overhead, but never a disk ratio because disk is unchanged. Both lifecycles must satisfy constraint~7 for every public entry. They use the same kernel but differ in deletion scope, publication, auditing, and fallback (Section~\ref{sec:lifecycle}).


\section{Execution-Aware Compression Theory}
\label{sec:theory}

\subsection{From Repetition to Scoped Reuse}

Within one document, \base selects a shortest cover of typed contract units using primitive statements, shared rules, parameterized procedures, and explicit exceptions. For candidate representation $z$, let $d(z)$ be its token cost and $\Gamma(z)$ the units it covers. The file-level problem is a weighted set cover with hard coverage:
\begin{equation}
\min_{Z\subseteq\mathcal{Z}}\sum_{z\in Z}d(z)
\quad\text{s.t.}\quad
\bigcup_{z\in Z}\Gamma(z)\supseteq\contract_v.
\label{eq:filecover}
\end{equation}
\method retains this optimizer but changes candidate cost according to where the representation is placed in the bundle.

Suppose exact fragment $x$ of length $d_x$ occurs in files $S_x\subseteq V$. Keeping all copies incurs $\sum_{v\in S_x}w_vd_x$, where $w_v$ is the effective load weight induced by Eq.~\eqref{eq:objective}. Factoring $x$ into \kept{shared module} $h$ with reference cost $d_{\mathrm{ref}}$ costs
\begin{equation}
w_hd_x+\sum_{v\in S_x}w_vd_{\mathrm{ref}}+d_{\mathrm{audit}},
\label{eq:sharingcost}
\end{equation}
where $w_h$ is the probability-weighted scope of $h$, and $d_{\mathrm{audit}}$ covers import instructions and provenance. Factoring is allowed only if Eq.~\eqref{eq:sharingcost} is smaller.

\begin{proposition}[No sparse-root promotion]
\label{prop:sparseroot}
Let $x$ occur only in branches with total access probability $p<1$. Placing $x$ in the always-loaded root adds $(1-p)d_x$ expected tokens relative to keeping one copy within the affected \hl{activation scope}, before reference overhead. Therefore root promotion is suboptimal whenever the deployment saving is smaller than $(1-p)d_x/\lambda$ plus reference cost.
\end{proposition}
\begin{proof}
In Eq.~\eqref{eq:objective}, root placement gives load weight one, whereas branch-scoped placement gives weight $p$. Their path-cost difference is $(1-p)d_x$; deployment can offset it only through the $\lambda C_{\mathrm{dep}}$ term. Adding nonnegative reference overhead yields the stated condition.
\end{proof}

The proposition captures a common failure: global deduplication can reduce storage while making common requests more expensive. \method instead places a shared module within the activation scope of the branches that need it. The root imports it only when every reachable path requires it.

\subsection{Conditional Capsules}

A guarded section with an explicit trigger can be moved from an always-loaded file into an on-demand \guard{capsule}. Let the section body cost $d_b$, dispatcher cost $d_g$, and trigger probability $p_g$. Keeping the body inline costs $d_b$ per load; a capsule costs $d_g+p_gd_b$ in path expectation, plus $\lambda d_g$ deployment overhead.

\begin{proposition}[Capsule threshold]
\label{prop:capsule}
Moving a guarded section to a capsule decreases Eq.~\eqref{eq:objective} if
\begin{equation}
(1-p_g)d_b>(1+\lambda)d_g.
\label{eq:capsulethreshold}
\end{equation}
\end{proposition}
The condition favors long, infrequent branches with short and unambiguous dispatchers. \method never infers a new guard: a capsule candidate must originate from an explicit heading or conditional clause, and the dispatcher must preserve the original trigger and relative path.

Capsules and sharing address different costs. A capsule delays an infrequent branch; a shared module removes exact repetition across branches. When both apply, the shared module remains within the union of those branches rather than moving to the root.

\subsection{Host Entailment Under a Closed Contract}

Some skill text restates guarantees already enforced by the deployment environment: an available tool, an immutable output schema, or a mandatory safety policy. Removing such text can yield a high activation reduction, but only if the guarantee is explicit and stable. We represent the environment as a signed typed contract
\begin{equation}
H=\{(t,k,v,\sigma,\delta)\},
\end{equation}
where $t$ is unit type, $k$ a canonical key, $v$ the exact value, $\sigma$ its scope, and $\delta$ the environment digest. Removal requires an exact type/key/value/scope match; semantic similarity only flags manual review.

\begin{assumption}[Environment stability]
The audited environment digest remains unchanged between compression and deployment. If it changes, the bundle is re-audited or the entailment transformation is disabled.
\end{assumption}

Absent a supplied contract, host entailment is a no-op. This default avoids treating model knowledge, tool documentation, or informal conventions as guarantees.

\subsection{Interface Contracts and the Witness Hierarchy}
\label{sec:witnesses}

Typed coverage and host entailment do not fully protect two structures that matter in deployment: interface contracts and evidence for deletion.

\begin{definition}[Interface contract]
An \emph{interface contract} is a maximal source section that specifies the skill's inputs or outputs as a machine-checked format: an output schema with worked examples, a whitelist of legal labels, or a field-level validation rule. Interface contracts are \emph{atomic}: the specification is read as one unit by the consuming model, so its coverage obligation is discharged only when the whole section is emitted contiguously and verbatim. A rewriting that preserves every line but scatters the section across several synthetic headings does \emph{not} cover the unit.
\end{definition}

The second gap concerns \emph{how} a removal is justified. Coverage by itself is a claim; a witness is the evidence attached to it. Every removal \method commits carries exactly one, recorded in the audit state:
\begin{equation}
W_1 \;\succ\; W_2 \;\succ\; W_3,
\label{eq:witnesshierarchy}
\end{equation}
where $W_1$ is \emph{literal containment}: the removed text survives byte-for-byte at another reachable location, so the removal is reversible by construction; $W_2$ is a \emph{cover witness}: a merged or shared representation passes the deterministic content-word coverage gate while retaining every protected literal and every negation polarity; and $W_3$ is an \emph{entailment witness}: a frozen checker model (temperature $0$, fixed prompt) certifies that the decision-relevant content of the removed span---conditions, verdicts, thresholds, whitelists, exemptions---is already fully expressed elsewhere. $W_3$ verdicts are accepted only for evidence-class units that carry no prohibition, output, or exemption marker, and each accepted verdict is logged with the checker's stated basis, so every $W_3$ removal is individually auditable after publication. Any deletion that cannot attach a witness is refused, whatever the predicted saving. Two further rules keep $W_3$ honest when it is applied inside a single document: the entailment base excludes the candidate segment itself, and two candidates may not justify each other's removal---a deletion whose only witness is another deletion is vetoed, because each verdict may have relied on text that is itself about to disappear. Where placement costs differ by layer, $W_3$ should additionally be spent where the executor charges most: a removal from the always-loaded root saves its weight once per reasoning round, while a removal from an on-demand reference saves it about once.

The hierarchy makes the safety boundary testable: \emph{witness strength, not deletion count, sets the safe compression ceiling}. Section~\ref{sec:industrial} compares the same bundle under $W_1,W_2$ alone and with restricted $W_3$ witnesses.

\subsection{Bundle Faithfulness Invariant}

Let $\Phi$ be the ordered transformation sequence: resolve, extract, drop entailed units, factor exact structure, form capsules, apply file-level shortest cover, materialize, and audit.

\begin{proposition}[Phase-A preservation]
Under a safe source graph, exact extraction provenance, Assumption~1, interface-contract atomicity, and a sound final auditor, a committed output $B'=\Phi(B)$ is faithful according to Definition~2 and can be executed by any harness that already supports the original bundle's ordinary relative-file loading semantics.
\end{proposition}
\begin{proof}[Proof sketch]
Host deletion retains a digest-bound witness. Scoped sharing replaces covered units with a reachable reference in the same activation scope. Capsule extraction preserves the explicit guard in a dispatcher, and file-level compression obeys Eq.~\eqref{eq:filecover}. Interface contracts remain contiguous, while locked nodes are copied byte-for-byte. The disk audit rejects uncovered units, unsafe or dangling paths, scope expansion, scattered contracts, and changed locked bytes. Every committed candidate thus meets Definition~2 and runs through the source model's ordinary files and relative paths, with no added runtime mechanism.
\end{proof}

As with any static natural-language verifier, the guarantee depends on the extractor and auditor. We reduce this trusted surface by retaining source spans, preferring exact transformations, auditing emitted files rather than an in-memory plan, and falling back to a byte-identical copy on uncertainty.

\subsection{Continual Compression}

For an incoming evolution patch $\Delta_t$, Continual Bundle Compression identifies an affected closure $A_t$: changed nodes, reference ancestors, newly or formerly referenced descendants, shared modules whose support set intersects the patch, capsules whose guards changed, and host witnesses invalidated by a new environment digest. Only this closure is re-extracted and locally re-optimized; unchanged contracts are reused by source digest. The graph-update cost is
\begin{equation}
O(|A_t|+|E(A_t)|)+\operatorname{LM}(T_t),
\label{eq:incrementalcost}
\end{equation}
where $T_t\subseteq A_t$ are changed textual nodes requiring extraction. A global audit remains $O(|V_t|+|E_t|)$ in the conservative implementation, but uses hashing and cached contracts rather than new language-model calls.

Local decisions may drift from the one-shot optimum as support sets and path frequencies evolve. Define continual regret at checkpoint $t$ as
\begin{equation}
R_t=\frac{J(B_t^{\mathrm{cont}})-J(B_t^{\mathrm{one}})}
{\max(1,J(B_t^{\mathrm{one}}))},
\label{eq:continualregret}
\end{equation}
where $B_t^{\mathrm{one}}$ is a full one-shot recompression of the same verbatim authored state. A triggered global repack resets placement debt. The trigger is conservative: recoverable saving above $\theta_{\mathrm{repack}}$, relative bundle growth above $\rho$, workload drift above $\delta_W$, an environment-digest change, or a maximum of $K$ patches. This gives a tunable continuum between full recompression after every write and cheap local updates.

Faithfulness does not depend on the schedule: both modes materialize and audit a complete directory. If continual optimization fails, the system publishes the verbatim patched bundle, not the previous compressed version. Compression may lose savings, but it cannot discard the patch.

\section{SkillZip Pro}
\label{sec:method}

Figure~\ref{fig:overview} summarizes \method. Like \base, it is evaluation-free and seeks the shortest faithful representation. The difference is scope: \base compiles one \texttt{SKILL.md}, whereas \method compiles a progressively loaded directory as a typed resource graph. It supports both full One-Shot optimization and state-reusing Continual updates.

The compiler removes linked-file content already supplied by the root or environment, then verifies every route. One-Shot or Continual determines when optimization runs; Persistent or Transient determines whether output replaces the canonical bundle or forms a per-run view.

\begin{figure*}[t]
  \centering
  \includegraphics[width=0.99\textwidth]{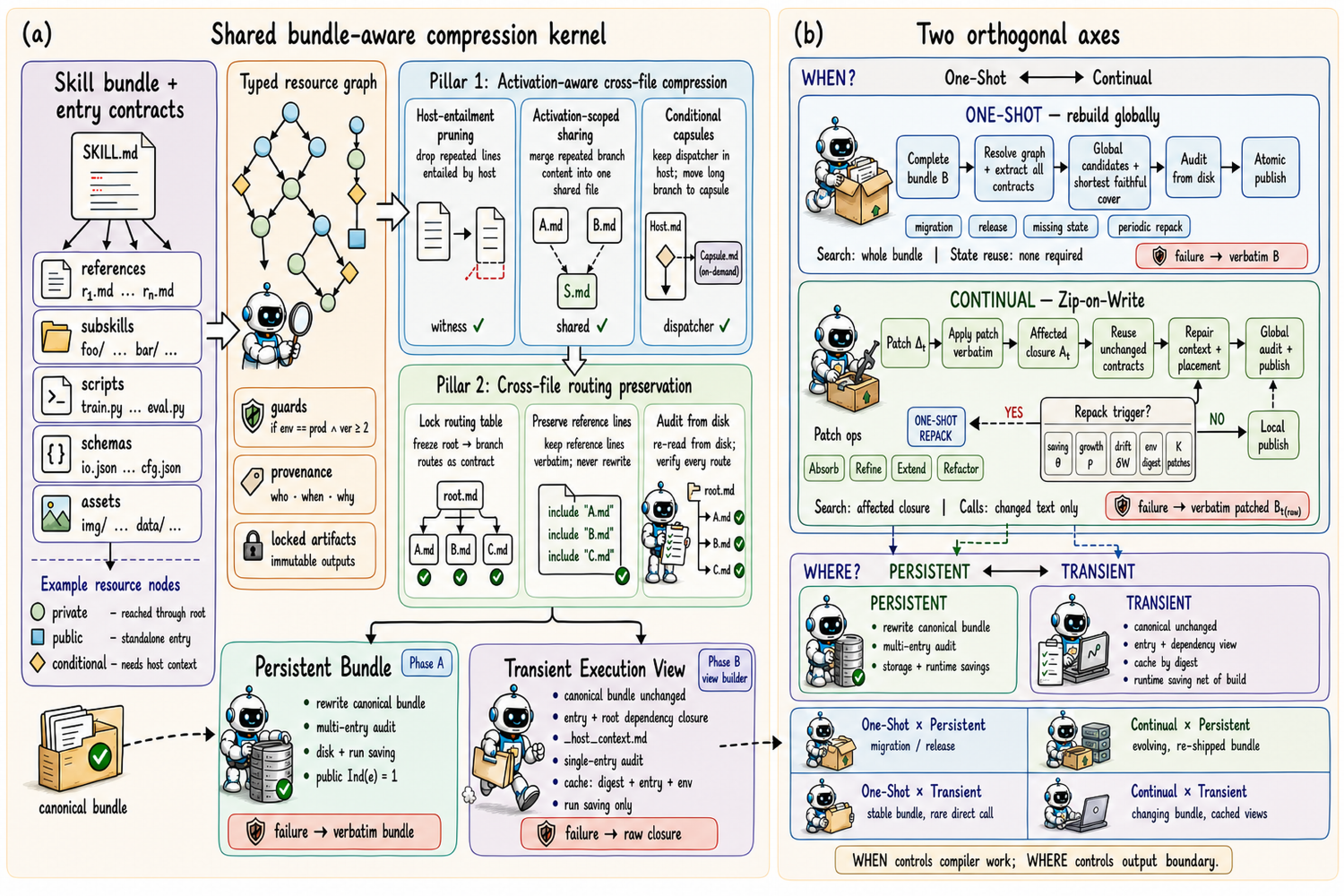}
  \caption{\textbf{Overview of SkillZip Pro.} \textbf{(a) Compression modes.} The bundle-aware kernel compresses a typed resource graph of skills, tools, prompts, policies, metadata, and inter-skill relationships while preserving routing, dependency, entry, and execution contracts. It supports global One-Shot compression, which builds a compact bundle in one pass, and state-reusing Continual compression, which incrementally updates an existing state as resources, tasks, or feedback change. \textbf{(b) Bundle forms.} Each mode produces either a Persistent Bundle for storage and repeated reuse, or a Transient Execution View for lightweight, task-specific execution. This separation balances compression, efficiency, adaptability, and contract fidelity.}
  \label{fig:overview}
\end{figure*}

\subsection{From SkillZip to SkillZip Pro}

Compared to \base, \method adds graph-wide placement, accounting, and audits.


\subsection{The Two Pillars}
\label{sec:pillars}

The extension has two goals. \textbf{\textcolor{SkillZipBlue}{Pillar~1 (compress across files)}} chooses what to remove and where to place what remains. \textbf{\textcolor{PreprintGreen}{Pillar~2 (preserve routing)}} ensures that these rewrites do not sever the links used for progressive loading. The first creates savings; the second makes them safe to deploy.

\textbf{\textcolor{SkillZipBlue}{Pillar 1 --- activation-aware cross-file compression.}} Three graph-level transformations, none available to a single-document compressor, act on top of the per-file optimizer. (i)~\emph{Host-entailment pruning}: a unit whose obligation is exactly implied by the root or by a declared environment contract is removed, with a digest witness recorded, because the agent already receives it. (ii)~\emph{Activation-scoped sharing}: text repeated across branches is factored into one \kept{shared module}, but placed \emph{only within the activation scope that loads it}---never promoted into the always-loaded root, which would raise per-run cost. (iii)~\emph{Conditional capsules}: a long guarded branch moves into an on-demand file, leaving a one-line dispatcher. Together, these transformations remove redundancy between the root and dynamically loaded files.

\textbf{\textcolor{PreprintGreen}{Pillar 2 --- cross-file routing preservation.}} The same rewriting is dangerous precisely because a routing line---``when the task matches $X$, read \texttt{[X](refs/x.md)}''---looks like ordinary prose but is the skill's \win{navigation table}. \method protects it with three layers of defense. (i)~The routing table is \guard{locked} as a first-class unit, so the optimizer may neither reword nor merge it. (ii)~Every reference-bearing source line is preserved through rendering, so a link cannot be dropped as a side effect of compressing the text around it. (iii)~Before publishing, an independent pass re-reads the materialized directory and verifies that \emph{every} file---including shared modules and capsules the compressor itself created---is still \win{reachable} from the root by following resolved links; if any branch became unreachable, the candidate is rejected and the verbatim bundle is shipped instead. Section~\ref{sec:loadingresults} measures the result: \method keeps $1.000$ routing fidelity while root-only and flat compressors drop to $0.000$.

\subsection{Two Deployment Lifecycles for the Kernel Output}
\label{sec:lifecycle}

Together, the pillars form a compression kernel that rewrites files while preserving their routes. A separate lifecycle choice determines where the result lives. This choice is independent of the One-Shot/Continual schedule in Sections~\ref{sec:oneshot}--\ref{sec:continual}: the schedule controls when the kernel runs, while the lifecycle controls whether its output replaces the shipped bundle.

\textbf{\hl{Persistent Bundle Compression}} runs the kernel on the canonical bundle and ships the smaller bundle in its place. It suits private references and subskills that are only ever reached from the root: their repeated text is covered elsewhere, so it can be removed once and for all, cutting both the shipped size and every future run that would have loaded it. Because the shipped files change, the kernel may touch a public entry only after a \emph{multi-entry audit} re-reads the published directory from \emph{every} declared public entry and confirms that each still routes correctly and keeps $\operatorname{Ind}=1$ on its own; if any check fails, the verbatim bundle is shipped. This is the lifecycle of every earlier section, and we label it as such from here on.

\textbf{\hl{Transient Execution-View Compression}} leaves the canonical bundle byte-for-byte unchanged and builds a throwaway \emph{execution view} $\widehat{B}_{e,\pi}$ before a run. For selected entry $e$, the view combines the closures of the root, $e$, and their shared dependencies, compresses them with the same kernel, and re-roots the result at $e$. The original root remains linked as \texttt{\_host\_context.md}, so host obligations survive. The view is discarded or cached by $(\text{bundle digest},\,e,\,\text{environment digest})$ until a key changes. It lowers run context, not disk use. Algorithm~\ref{alg:transient-view} summarizes the transient compression. If compression does not improve the view or its single-entry audit fails, the system loads the uncompressed closure. 


\begin{algorithm}[t]
\caption{Transient Execution-View Construction}
\label{alg:transient-view}
\begin{algorithmic}[1]
\Require canonical bundle $B$ with root $r$, chosen entry $e$, optional environment $H$, cache flag
\Ensure execution view rooted at $e$; canonical $B$ unchanged
\State $k\gets\Call{ViewKey}{\operatorname{digest}(B),e,\operatorname{digest}(H)}$
\If{cache \textbf{and} $\Call{Hit}{k}$}
  \State \Return $\Call{LoadCached}{k}$ \Comment{no kernel call}
\EndIf
\State $\Omega\gets\Call{Reach}{B,e}\cup\Call{Reach}{B,r}$ \Comment{entry and root closures}
\State $R\gets\Call{MaterializeClosure}{B,\Omega,e}$ \Comment{promote $e$ to root; link $r$ as host context}
\State $\widehat{B}\gets\Call{Kernel}{R}$ \Comment{Pillars~1--2, cross-file promotion on}
\If{$\widehat{B}$ unusable \textbf{or} $\tau(\widehat{B})\ge\tau(R)$ \textbf{or} single-entry audit fails}
  \State $\widehat{B}\gets R$ \Comment{safe fallback: uncompressed closure}
\EndIf
\State \Return $\Call{Cache}{k,\widehat{B}}$ \Comment{verify $\operatorname{digest}(B)$ unchanged}
\end{algorithmic}
\end{algorithm}

Because \emph{when} to compress and \emph{whether to keep the result} are independent, they combine into the four production modes of Table~\ref{tab:lifecycle-matrix}. One-shot persistent is the initial migration of an evolved library into a shipped bundle; continual persistent keeps that bundle small as it is edited and re-shipped; one-shot transient serves a stable bundle whose entries are called directly and rarely; and continual transient serves high-frequency direct calls against a bundle that keeps changing, where digest-keyed caching amortizes the build over many runs. 


\begin{table*}[t]
\centering
\caption{\textbf{The two axes are orthogonal.} Update frequency (One-Shot vs.\ Continual) decides \emph{when} the kernel runs; lifecycle (Persistent vs.\ Transient) decides \emph{whether its output replaces the shipped bundle}. The four combinations target different deployments.}
\label{tab:lifecycle-matrix}
\resizebox{\textwidth}{!}{%
\begin{tabular}{llp{0.30\textwidth}p{0.28\textwidth}p{0.26\textwidth}}
\toprule
Lifecycle & Compression mode & Best for & What it saves & System requirement \\
\midrule
\multirow{2}{*}{\textbf{Persistent}} & One-Shot & initial migration of an evolved library to a shipped bundle & smallest canonical bundle and lowest steady-state run cost & one multi-entry audit before publishing \\
 & Continual & a library that keeps evolving and is re-shipped over time & each edit is repacked once and amortised over many runs & re-audit on every write plus a periodic global repack \\
\midrule
\multirow{2}{*}{\textbf{Transient}} & One-Shot & a stable bundle whose entries are called directly and rarely & per-run context saving with the canonical bundle untouched & a view build (or cache read) before each run \\
 & Continual & high-frequency direct calls against a frequently edited bundle & views are cached by bundle digest and reused across runs & digest-keyed cache with invalidation on any edit \\
\bottomrule
\end{tabular}}
\end{table*}

\subsection{Shared Inputs and Persistent Compiler State}

The required input is a bundle root and its entry file. Optional inputs are a deployment tokenizer or cost callback, a workload ledger of task-class frequencies and observed resource loads, and a typed, digest-bound environment contract. Without traces, explicit branch structure estimates path weights. Without an environment contract, no host-entailment deletion is attempted. Both modes share a persistent compressor state
\begin{equation}
M_t=\langle G_t,\contract_t,W_t,H_t,\mathcal{I}_t,\mathcal{P}_t,\mathcal{A}_t\rangle,
\label{eq:compilerstate}
\end{equation}
where $G_t$ is the safe resource graph; $\contract_t$ the per-node contract store; $W_t$ the path-weight ledger; $H_t$ the environment digest and witnesses; $\mathcal{I}_t$ candidate, support, and reverse-reference indices; $\mathcal{P}_t$ provenance; and $\mathcal{A}_t$ prior audit results. The sidecar accelerates future compression but is never required at execution. If it is absent or its digest is stale, the tool safely falls back to one-shot reconstruction.

\subsection{Shared Bundle Compiler}

\textbf{Safe graph resolution.}
The scanner walks the directory without following escaping symbolic links and classifies each node as instructional Markdown, subskill entry, code, structured data, or opaque asset. It recognizes local Markdown links, explicit path literals with known extensions, and declarative subskill references. Every internal edge records its exact source span and nearest explicit guard. External URLs are preserved but never fetched; missing, ambiguous, or escaping local targets fail strict mode. Unreferenced files remain in deployment cost and are never silently discarded.

\textbf{Typed contract extraction.}
Each instructional node is segmented by headings and source spans. One structured extraction recovers Eq.~\eqref{eq:contract}, while every unit retains provenance. Relations such as \texttt{same\_rule}, \texttt{same\_workflow}, and \texttt{exception\_of} are accepted only after deterministic type, polarity, guard, and payload checks. Low-confidence units are locked and preserved verbatim. The original \base file-level scanner, type-compatible reuse, minimum-cost cover, fixed-template render, and span-restoration audit operate here unchanged for each eligible file.

\textbf{Bundle-level candidate generation.}
Three transformations expose savings unavailable inside a single file. (1) Exact contract units entailed by $H_t$ may be removed only with a scope-compatible, digest-bound witness. (2) Repeated exact rules or workflow fragments may move to \texttt{.skillzip\_shared/<digest>.md} at their lowest safe activation scope; affected branches receive a mandatory relative loading instruction. (3) A long section with an explicit guard may move to \texttt{capsules/<slug>.md}; a dispatcher retaining the trigger remains at the original site. Each transformation is evaluated together with per-file covers under Eq.~\eqref{eq:objective}.

\textbf{Global selection and materialization.}
Because a capsule changes path weights and a shared module may serve several capsules, marginal selection is followed by local add/drop/swap improvement. Every move preserves modeled coverage. Materialization writes an ordinary relative-path directory; code, data, schemas, and assets remain byte-identical. Interface contracts (Definition~3) are restored as one contiguous verbatim span: surrounding prose may shrink, but the optimizer cannot split, reorder, or paraphrase the contract. A post-pass removes scattered render fragments and restores the source span. Every removal records its witness class from Eq.~\eqref{eq:witnesshierarchy}; $W_3$ removals also store the verdict and rationale.

\subsection{Mode I: \hl{One-Shot} Bundle Compression}
\label{sec:oneshot}

One-shot mode is intended for initial migration, release packaging, a missing or invalid state sidecar, or a periodic global re-optimization after workload drift. It reconstructs $M_t$ from the complete source bundle and searches globally. Algorithm~\ref{alg:skillzippro-oneshot} never publishes a candidate whose materialized objective is worse than the source.

\begin{algorithm}[t]
\caption{SkillZip Pro: One-Shot Bundle Compression}
\label{alg:skillzippro-oneshot}
\begin{algorithmic}[1]
\Require source bundle $B$, entry $r$, optional $H,Q$
\Ensure audited bundle $B^{\star}$ and state $M$, or verbatim $B$
\State $G\gets\Call{ResolveSafeGraph}{B,r}$
\State $\contract\gets\Call{ExtractTypedContracts}{G}$
\State $W\gets\Call{EstimateLoadWeights}{G,Q}$
\State $Z\gets\Call{FileCoverCandidates}{\contract}$
\State $Z\gets Z\cup\Call{HostEntailments}{\contract,H}$
\State $Z\gets Z\cup\Call{ScopedSharing}{G,\contract,W}$
\State $Z\gets Z\cup\Call{GuardedCapsules}{G,\contract,W}$
\State $P\gets\Call{ConstrainedSelect}{Z,J,\operatorname{cover}}$
\State $D\gets\Call{MaterializeTemporary}{B,P}$
\If{$J(D)\le J(B)$ \textbf{and} $\Call{AuditFromDisk}{B,D,H}$}
  \State \Return $\Call{AtomicCommit}{D},\Call{SaveState}{D}$
\Else
  \State \Return $\Call{VerbatimCopy}{B},\Call{FailureState}{B}$
\EndIf
\end{algorithmic}
\end{algorithm}

\subsection{Mode II: \hl{Continual} Bundle Compression}
\label{sec:continual}

Continual mode targets self-evolving skills that receive frequent small patches. Let $\Delta_t$ contain added, modified, moved, and deleted paths. The tool first applies $\Delta_t$ without paraphrase to the current authored source, producing the semantic fallback $B_t^{\mathrm{raw}}$. It then updates reference edges and constructs an invalidation closure $A_t$ containing changed nodes, their reference ancestors, newly or formerly referenced descendants, shared modules whose support changed, capsules whose guard changed, and any host witness invalidated by an environment digest change.

Only changed textual nodes are re-extracted; digest-identical contracts are reused. Patch units retain the four operations of \base: \textsc{Absorb} removes an exact duplicate already covered; \textsc{Refine} updates an existing rule, guard, argument, or exception; \textsc{Extend} adds genuinely new behavior or a new resource; and \textsc{Refactor} changes a representation when several edits make another cover shorter. Pro extends \textsc{Refactor} to placement: affected shared modules, capsules, and host witnesses are re-priced whenever their support set or path weight changes.

Local repair cannot accumulate unbounded debt. A one-shot global repack is triggered when recoverable objective saving exceeds $\theta_{\mathrm{repack}}$, bundle growth exceeds $\rho$, the workload distribution drifts beyond $\delta_W$, the environment digest changes, or $K$ patches have elapsed. Repacking reads the current bundle and compact state, never the full patch history.

\begin{algorithm}[t]
\caption{SkillZip Pro: Continual Bundle Compression}
\label{alg:skillzippro-continual}
\begin{algorithmic}[1]
\Require authored source $S_{t-1}$, audited state $M_{t-1}$, patch $\Delta_t$
\Ensure current faithful bundle $B_t$ and state $M_t$
\State $S_t,B_t^{\mathrm{raw}}\gets\Call{ApplyPatchVerbatim}{S_{t-1},\Delta_t}$
\State $G_t,A_t\gets\Call{UpdateGraphAndClosure}{M_{t-1}.G,S_t,\Delta_t}$
\State $\contract_t\gets\Call{ReuseAndReextract}{M_{t-1}.\contract,A_t}$
\State $U_t\gets\Call{Classify}{\textsc{Absorb},\textsc{Refine},\textsc{Extend},\textsc{Refactor}}$
\State $Z_t\gets\Call{RefreshAffectedCandidates}{A_t,U_t,M_{t-1}}$
\If{$\Call{RepackTriggered}{M_{t-1},S_t,Z_t}$}
  \State \Return $\Call{OneShot}{S_t}$
\EndIf
\State $P_t\gets\Call{RepairLocalCoverAndPlacement}{Z_t,J}$
\State $D_t\gets\Call{MaterializeTemporary}{B_t^{\mathrm{raw}},P_t}$
\If{$J(D_t)\le J(B_t^{\mathrm{raw}})$ \textbf{and} $\Call{AuditFromDisk}{S_t,D_t,H_t}$}
  \State \Return $\Call{AtomicCommit}{D_t},\Call{SaveState}{D_t}$
\Else
  \State \Return $\Call{AtomicCommit}{B_t^{\mathrm{raw}}},\Call{RebuildState}{S_t}$
\EndIf
\end{algorithmic}
\end{algorithm}

The fallback distinction is essential. A failed one-shot run may return its input. A failed continual run cannot reactivate $B_{t-1}$ after a valid patch, which would discard learned behavior; it publishes $B_t^{\mathrm{raw}}$ and records the failure for retry.

\subsection{Cross-File Audit and Mode Selection}

The auditor ignores in-memory coverage claims and rereads the emitted directory. It checks (i) graph closure and path confinement; (ii) typed coverage at a compatible scope or by an exact host witness; (iii) mandatory reachability and guard preservation for every capsule and shared module; (iv) path and SHA-256 identity for locked artifacts; (v) that every interface contract appears as one contiguous span whose normalized lines equal the source section, with no scattered fragments elsewhere; and (vi) all four materialized costs. Atomic publication occurs only after audit succeeds.



\subsection{Harness-Agnostic Deployability}

Both modes remain agent harness-agnostic. The compressor may persist state, observe offline traces, and run on every write, but the published bundle does not depend on it. The agent selects the same skill, loads the same root entry, follows ordinary relative links, and executes byte-identical artifacts. 


\section{Experiments}
\label{sec:experiments}

We design our experiments around three goals: determining whether \method reduces both deployment and progressively loaded runtime costs, verifying that these savings preserve task quality, routing, and the independent usability of public entries, and assessing whether its four operating modes remain practical as skills evolve. We evaluate these questions on three agent benchmarks, real self-evolving skill libraries, a controlled multi-entry bundle, and a production content-moderation skill. Comparisons with root-only, flattened, evaluation-guided, and expert-designed baselines are complemented by ablations and workload sweeps that isolate the contribution and cost of each design component.


\subsection{What We Ask}

\begin{itemize}
  \item \textbf{Q1: Is one ratio enough?} When does a root-only ratio misrepresent the four cost layers?
  \item \textbf{Q2: Does it help?} Does \method cut all four costs while keeping task success?
  \item \textbf{Q3: Does it still load correctly?} Does the bundle open required files and skip irrelevant ones?
  \item \textbf{Q4: Is the saving real?} How much source knowledge remains reachable rather than deleted?
  \item \textbf{Q5: Does it port across models?} Does one compressed bundle stay useful when a different model family runs it?
  \item \textbf{Q6: What resources does compression consume?}
  \item \textbf{Q7: Which part does the work?} How much comes from the resource graph, host entailment, scoped sharing, capsules, and the audit?
  \item \textbf{Q8: One-shot or continual?} How much update work does continual mode save, how far does it drift from a full rebuild, and when should it repack?
  \item \textbf{Q9: Does it hold in production?} On a deployed skill, does witnessed compression preserve quality, and do modeled savings match runtime measurements?
\end{itemize}

\subsection{Benchmarks and Skill Bundles Construction}

We evaluate \method on three task families: \textbf{BFCL-v4} for tool use~\cite{patil2025bfcl,bfcl2026v4}, \textbf{LiveMathematicianBench} for mathematical problem solving~\cite{he2026livemath,he2026livemathematicianbench}, and \textbf{SpreadsheetBench} for spreadsheet editing~\cite{ma2024spreadsheetbench}. These benchmarks require different kinds of instructions: BFCL-v4 depends on correct tool names and call order, LiveMathematicianBench requires detailed reasoning and verification, and SpreadsheetBench requires precise cell values and formatting.

The skill bundles are generated through self-evolution in our experiments. For each benchmark, we run the skill evolution optimizer \textsc{SkillOpt}~\cite{yang2026skillopt} separately on every task class. The optimizer proposes one edit at a time and retains it only when it improves performance on that class. Each task therefore develops a specialised skill containing its own procedures, checks, and failure experience. These skills are not interchangeable, making routing important: the agent must identify and load the file associated with the current task class. From the resulting skills, we replay evolution patches into separate named branches. A deterministic builder records the source of every moved section and preserves its text verbatim. This process reproduces the structure that develops during continued self-evolution: common output rules, verification checklists, and records of previous errors are copied into multiple branches, while one branch accumulates a longer set of guarded edge cases.



We divide tasks into non-overlapping evolution, validation, and held-out test sets. Skill construction and compression use only the data assigned to evolution; the held-out test set is used exclusively for final evaluation. Every compression method receives the \emph{same input bundle} and the \emph{same environment}.

\subsection{Models and Setup}

We use the same three model families as the earlier manuscript: Qwen3.7-Max, Qwen3.6-Plus, and Kimi K2.6~\cite{qwen2026qwen37,qwen2026qwen36,moonshot2026kimik26}. Qwen3.6-Plus runs the main tables, and all three serve as executors in the portability study. Sampling settings, tools, system prompt, and timeouts are fixed per benchmark. Deterministic compression needs one run per bundle; repeats are byte-identical.

The executor is the unmodified benchmark agent. It sees catalog metadata first, then the root once the skill is selected. Other files appear only after an ordinary file read, and a file can be opened only if something already loaded links to it. A wrapper records which paths were opened and how many tokens were read; it never injects, hides, or reorders content.

\subsection{Compared Conditions and Baselines}

We compare \method against two reference conditions and four compression baselines. Unless stated otherwise, all compression methods receive the same evolved bundle, and all reported savings are computed relative to the uncompressed.

\paragraph{Reference conditions.}
\textbf{No Skill} runs the agent without reusable instructions and provides a lower reference for task performance. \textbf{Human Skill} uses the original hand-written skill before self-evolution. \textbf{Evolved Bundle} is the complete, uncompressed output of the evolution process and serves as the primary reference for both fidelity and cost.

\paragraph{Compression baselines.}
\textbf{Root-only SkillZip} applies \base only to the root \texttt{SKILL.md} and leaves all referenced resources unchanged, representing single-document compression in a multi-file setting. \textbf{Flat-concat SkillZip} concatenates all textual resources with path delimiters, compresses the resulting document, and then maps the output back to files; it tests the effect of ignoring progressive-loading boundaries. \textbf{SkillReducer} is the evaluation-guided skill compressor. To match our no-rollout compression budget, we disable its evaluation-based candidate selection and retain only its compression stage. \textbf{Expert Progressive} is a deterministic, structure-aware baseline that factors paragraphs repeated across two or more files into a shared resource referenced from the root.

\subsection{Reading the Baselines: One Warning First}
\label{sec:warning}

One baseline needs a warning before any number is read, because it otherwise looks like the strongest method in the paper. \textbf{Flat-concat \base{} posts the largest savings of any method and also throws away almost the entire skill.} On the real evolved libraries of Section~\ref{sec:evolved} it keeps \textbf{\emph{0.2\%}} of the skill's instruction lines: pasting every file together and compressing the result deletes the text and destroys the links that made the files reachable. Its savings are therefore \emph{not comparable} with the other rows, and we mark it with $\dagger$ wherever it appears. We keep it in the tables on purpose: it is the clearest demonstration that \textbf{\emph{a compression ratio means nothing until fidelity is measured next to it}}.

Root-only \base{} needs a smaller warning. It only rewrites \texttt{SKILL.md}, so it looks harmless, but that one file holds the routing list. Rewriting it keeps every file on disk while making the branches unreachable in practice, which is why its routing score is $0.000$ everywhere.



\subsection{Metrics}

\textbf{Task success.} Each benchmark's own automatic checker: boxed answer plus symbolic equality for math, normalized match against accepted answers for BFCL with its offline search tool available, and cell-by-cell workbook comparison for SpreadsheetBench, where a task passes only if every one of its test cases matches. For each method we also test whether it stays as good as the uncompressed bundle: we pool the held-out tasks of all three benchmarks, pair every method against the uncompressed bundle task by task, and bootstrap the difference in success rate (10{,}000 resamples). A method \emph{keeps quality} when the lower end of the 95\% interval stays above a five-point margin fixed before we looked at the numbers.

\textbf{Cost.} We measure the four layers of Section~\ref{sec:problem}---catalog, activation, deployment, and loaded path---plus objective $J$ from Eq.~\eqref{eq:objective}. Negative savings remain visible.

\textbf{Loading.} Whether the agent opened the file that specialises in the task (required recall), what share of opened files were not that file (irrelevant load), how many links are broken or point outside the bundle, and how many files became unreachable from the root.

\textbf{Knowledge kept.} We count the source instruction lines still present, possibly lightly reworded, in reachable files. Environment-guaranteed removals are recorded separately rather than counted as loss; each has a signed audit witness.

\textbf{Run cost.} Compression time, model calls, tokens, peak memory, and rollouts, measured end to end.

\subsection{Main Result: Task Success and the Four Costs}

Tables~\ref{tab:main-quality} and~\ref{tab:main-cost} must be read together, so that no method looks good on quality alone or on compression alone.

\begin{table}[!t]
\centering
\caption{Task success on held-out tasks, using the unmodified agent (Qwen3.6-Plus, grown version). Higher is better.}
\label{tab:main-quality}
\resizebox{\columnwidth}{!}{%
\begin{tabular}{lcccc}
\toprule
Method & BFCL~$\uparrow$ & Math~$\uparrow$ & Sheet~$\uparrow$ & Avg.~$\uparrow$ \\
\midrule
No Skill & 0.809 & 0.364 & 0.312 & 0.495 \\
Human Skill & 0.905 & 0.455 & 0.333 & 0.564 \\
Evolved Bundle & 0.905 & 0.364 & 0.354 & 0.541 \\
Root-only \base & 0.905 & 0.333 & 0.271 & 0.503 \\
Flat-concat \base & 0.857 & 0.394 & 0.167 & 0.473 \\
SkillReducer & 0.857 & 0.394 & 0.271 & 0.507 \\
Expert Progressive & 0.905 & 0.333 & 0.312 & 0.517 \\
\rowcolor{SkillZipPaleBlue}\method~(One-Shot) & 0.952 & 0.394 & 0.333 & 0.560 \\
\bottomrule
\end{tabular}}
\end{table}

\begin{table}[!t]
\centering
\caption{Four cost layers on three benchmarks. Cells report tokens and reduction from the Evolved Bundle; the agent is unchanged.}
\label{tab:main-cost}
\resizebox{\columnwidth}{!}{%
\begin{tabular}{lcccc}
\toprule
Method & Shipped~$\uparrow$ & Always loaded~$\uparrow$ & One run, avg~$\uparrow$ & One run, worst~$\uparrow$ \\
\midrule
Human Skill & 328/+87.7\% & 328/-51.4\% & 328/+43.3\% & 328/+54.3\% \\
Evolved Bundle & 2649/0\% & 212/0\% & 588/0\% & 755/0\% \\
\rowcolor{SkillZipPaleGrey}\color{LossyGrey}Root-only \base\,\xmark & 2589/+2.2\% & 153/+27.8\% & 529/+10.3\% & 696/+8.2\% \\
\rowcolor{SkillZipPaleGrey}\color{LossyGrey}Flat-concat \base\,\xmark & 954/+63.8\% & 120/+42.9\% & 183/+68.9\% & 249/+67.3\% \\
SkillReducer & 2518/+5.0\% & 213/-0.4\% & 564/+4.1\% & 722/+4.4\% \\
Expert Progressive & 2285/+13.7\% & 226/-6.5\% & 482/+17.9\% & 754/+0.2\% \\
\rowcolor{SkillZipPaleBlue}\method~(One-Shot) & 2112/+19.7\% & 196/+7.3\% & 500/+14.3\% & 669/+10.8\% \\
\bottomrule
\end{tabular}}
\end{table}

Three things stand out. First, a single ratio really can mislead: Root-only \base takes almost nothing off the shipped size yet shortens the always-loaded root a lot, while Human Skill looks like the best "compressor" by shipped size and at the same time makes the always-loaded root \emph{longer}. Second, Flat-concat \base posts by far the biggest numbers in both shipped size and per-run cost -- and Sections~\ref{sec:keptresults} show those numbers are paid for by throwing away the skill. Third, \method is the only method that cuts all four layers at once while matching the uncompressed bundle on task success. Figure~\ref{fig:costlayers} shows the per-layer picture.

On BFCL every method scores high because the offline search tool does much of the work; the skill still helps (no skill $0.809$, uncompressed $0.905$), and \method reaches the top score of any method here (\hl{$0.952$}), but the gaps are a few questions wide on twenty-one held-out questions, so we lean on the pooled test below rather than this single column. On SpreadsheetBench, where a task passes only if \emph{every} one of its test cases reproduces the gold range, \method scores \win{$0.333$} against \emph{$0.312$} for the hand-built expert baseline and sits one task in forty-eight below the uncompressed bundle -- we sample that benchmark at twice the density of the others precisely because its strict verdict makes single-task flips dominate at small $n$. Table~\ref{tab:qstats} runs the pooled test.

\begin{table}[!t]
\centering
\caption{Does the compressed bundle stay as good as the uncompressed one? Pooled over all \textbf{102} held-out tasks (33 math, 21 BFCL, 48 spreadsheet), paired task by task against the Evolved Bundle, 10{,}000 bootstrap resamples. A method keeps quality (\cmark) when the 95\% interval stays above the $-0.05$ margin fixed before the numbers were seen.}
\label{tab:qstats}
\resizebox{\columnwidth}{!}{%
\begin{tabular}{lccc c}
\toprule
Method & Success~$\uparrow$ & vs Evolved~$\uparrow$ & 95\% interval & Keeps quality \\
\midrule
No Skill & 0.431 & -0.039 & [-0.108, +0.029] & \xmark \\
Human Skill & 0.490 & +0.020 & [-0.039, +0.078] & \cmark \\
Root-only \base & 0.422 & -0.049 & [-0.128, +0.020] & \xmark \\
Flat-concat \base & 0.382 & -0.088 & [-0.167, -0.010] & \xmark \\
SkillReducer & 0.431 & -0.039 & [-0.108, +0.029] & \xmark \\
Expert Progressive & 0.441 & -0.029 & [-0.088, +0.029] & \xmark \\
\rowcolor{SkillZipPaleBlue}\method~(One-Shot) & 0.480 & +0.010 & [-0.029, +0.059] & \cmark \\
\bottomrule
\end{tabular}}
\end{table}

The test is decisive and answers Q2 directly. Across \hl{102} pooled held-out tasks, \method has the highest pooled success of any compression method, and it is the \win{only} compressor that keeps quality: its interval against the uncompressed bundle stays above the margin, while root-only, flat-concat, SkillReducer, and even the hand-built expert baseline all fall below it. In other words, every other way of making the bundle smaller also made the agent measurably worse, and \method did not.

\begin{figure}[!t]
\centering
\includegraphics[width=\columnwidth]{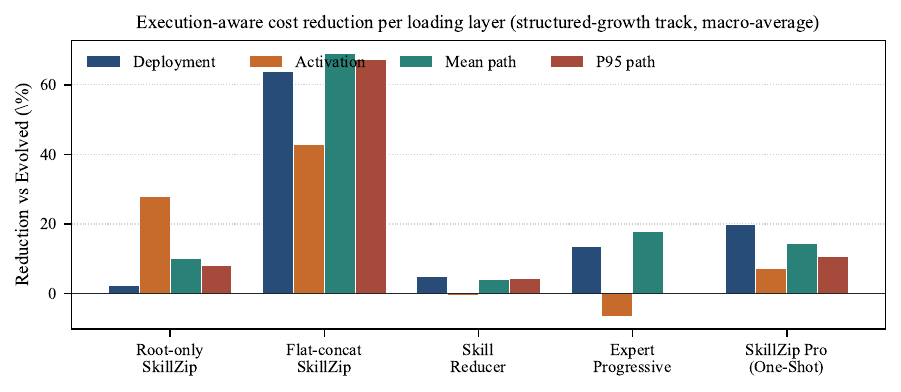}
\caption{Saving in each cost layer (grown version, averaged over benchmarks). Bars below zero mean that layer got \emph{worse}: a method can shrink the shipped package while making every single run longer, which one ratio would hide.}
\label{fig:costlayers}
\end{figure}

\begin{takeawaybox}
A single ratio is insufficient. Only \method reduces all four costs while matching the uncompressed skill; every other compressor regresses elsewhere.
\end{takeawaybox}


\subsection{Does Compression Preserve Progressive Loading?}
\label{sec:loadingresults}

A smaller bundle is useful only if the agent can still identify and load the resources required by each task. We therefore execute the grown bundles with the unmodified agent and measure both routing correctness and loading behavior. Required recall is the fraction of task-relevant resources that are loaded, irrelevant load is the fraction of loaded resources that the task does not need, and dispatch measures whether the agent selects the intended branch. We also record link errors, fallback activations, and the average number of reachable files.

\begin{table}[!t]
\centering
\caption{Loading behavior on grown bundles. Required recall is maximized; broken links, unreachable files, and irrelevant loads are minimized.}
\label{tab:loading}
\resizebox{\columnwidth}{!}{%
\begin{tabular}{lcccccc}
\toprule
Method & Req.\ recall~$\uparrow$ & Irrel.\ load~$\downarrow$ & Dispatch~$\uparrow$ & Link errors~$\downarrow$ & Fallback~$\downarrow$ & Files reached \\
\midrule
Evolved Bundle & 0.795 & 0.171 & 0.795 & 0 & -- & 6.3 \\
\rowcolor{SkillZipPaleGrey}\color{LossyGrey}Root-only \base\,\xmark & 0.685 & 0.290 & 0.685 & 0 & 0.000 & 6.3 \\
\rowcolor{SkillZipPaleGrey}\color{LossyGrey}Flat-concat \base\,\xmark & 0.708 & 0.259 & 0.708 & 0 & 0.000 & 6.0 \\
SkillReducer & 0.688 & 0.288 & 0.688 & 0 & 0.000 & 6.3 \\
\rowcolor{SkillZipPaleBlue}\method~(One-Shot) & 0.755 & 0.193 & 0.755 & \textbf{0} & 0.000 & 10.3 \\
\bottomrule
\end{tabular}}
\end{table}

Among the compression methods, \method most closely matches the loading behavior of the uncompressed bundle. It achieves a required-resource recall of $0.755$, compared with $0.795$ for the Evolved Bundle, while reducing irrelevant loading to $0.193$. The corresponding rates for the other compressors range from $0.259$ to $0.290$. The larger number of reachable files under \method reflects the shared modules and conditional capsules introduced during compression; these additional nodes remain connected to the appropriate execution paths rather than being loaded indiscriminately.

The absence of link errors alone does not establish routing fidelity. Root-only and Flat-concat \base remove routing instructions instead of leaving dangling links, so their branches remain on disk but can no longer be selected. The routing audit confirms this distinction: both baselines preserve $0.000$ of the original routing pairs, whereas \method preserves $1.000$. Its savings therefore come from shortening and sharing the resources that the agent loads, rather than making required branches unreachable.

\begin{takeawaybox}
\method preserves every declared route and retains loading behavior closest to the uncompressed bundle. Its runtime savings come from reducing the content loaded along valid execution paths, not from suppressing required branches.
\end{takeawaybox}

\begin{insightbox}{1}
For the compression of progressively loaded skills, the danger is not lost text but lost \emph{reachability}: rewriting only the root leaves every branch on disk yet unreachable in practice, so it looks safe and scores like deletion. Routing fidelity has to be measured separately from how many files still exist.
\end{insightbox}

\subsection{Is the Saving Real, or Just Deleted Text?}
\label{sec:keptresults}

A compressor can post a very large ratio simply by removing text, and a cost table cannot tell that apart from genuine compression. Table~\ref{tab:retention} therefore reports, for every method, how much of the original skill still reaches the agent, next to how much of the \emph{repeated} text it managed to remove. Repeated text is the part a bundle-level method is entitled to reclaim; it is the same sentence appearing in several files.

\begin{table}[!t]
\centering
\caption{Knowledge retained versus text saved on grown bundles. ``Kept'' counts original instruction lines that remain reachable; ``Lost'' counts unwitnessed deletions. $\dagger$ denotes an incomplete bundle.}
\label{tab:retention}
\resizebox{\columnwidth}{!}{%
\begin{tabular}{lccccc}
\toprule
Method & Kept~$\uparrow$ & Lost lines~$\downarrow$ & Repeated text removed~$\uparrow$ & Shipped saving~$\uparrow$ & Routing kept~$\uparrow$ \\
\midrule
Evolved Bundle & 1.000 & 0.0 & 0.000 & +0.0\% & 1.000 \\
\rowcolor{SkillZipPaleGrey}\color{LossyGrey}Root-only \base\,\xmark & 0.939 & 5.3 & 0.002 & +2.2\% & 0.000 \\
\rowcolor{SkillZipPaleGrey}\color{LossyGrey}Flat-concat \base\,\xmark & 0.241 & 59.0 & 1.000 & +63.8\% & 0.000 \\
SkillReducer & 1.000 & 0.0 & 0.073 & +5.0\% & 1.000 \\
Expert Progressive & 1.000 & 0.0 & 0.482 & +13.7\% & 1.000 \\
\rowcolor{SkillZipPaleBlue}\method~(One-Shot) & 0.986 & 1.0 & 0.736 & +19.7\% & 1.000 \\
\bottomrule
\end{tabular}}
\end{table}

The result reframes the whole comparison. Flat-concat \base appeared to be the strongest compressor by a wide margin, but it keeps under a quarter of the skill's instruction lines: its ratio is mostly deletion, not compression. Root-only \base also drops lines, because rewriting the root alone rewrites text that carried real conditions. \method removes the largest share of repeated text of any method while keeping essentially the whole skill, and its only removals that are not recoverable from the output are the ones the environment contract guarantees, each stored with a signed witness. Figure~\ref{fig:retention} plots the two quantities together.

\begin{figure}[!t]
\centering
\includegraphics[width=\columnwidth]{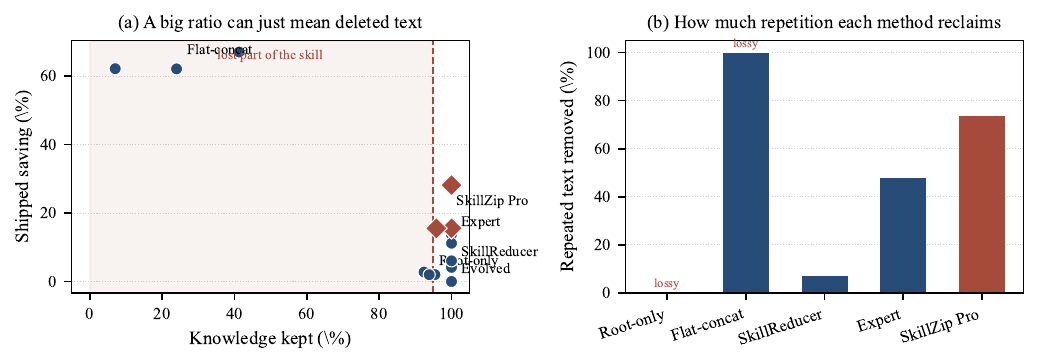}
\caption{(a) Saving versus knowledge retained; shaded points lost content. (b) Repetition reclaimed. \method achieves the largest faithful saving.}
\label{fig:retention}
\end{figure}

\begin{takeawaybox}
A large ratio may reflect deletion. \method removes the most repeated text while preserving the complete skill.
\end{takeawaybox}

\subsection{Case Study: What Does Compression Remove?}
\label{sec:example}

Aggregate compression ratios do not reveal whether a method removes redundancy or discards task-relevant content. We therefore examine one strategy file, \texttt{round\_02}, from the evolved \texttt{qwen3.6-plus} library and compare its treatment under Flat-concat \base and \method. This library contains substantial cross-file repetition: the same mathematical-output rules occur at 16 locations across the bundle, including all 15 strategy branches, while each branch retains its own workflow, verification steps, and conditional instructions. This example shows whether a compressor can eliminate the repeated content without changing the availability or behavior of the individual branch.

\begin{figure*}[!t]
\centering
\footnotesize

\begin{minipage}[t]{0.30\textwidth}
\begin{tcolorbox}[
  enhanced,
  colback=white,
  colframe=LossyGrey,
  boxrule=0.6pt,
  arc=1.5pt,
  title={\bfseries\sffamily\small Original Branch},
  left=3pt,
  right=3pt,
  top=2pt,
  bottom=2pt,
  fonttitle=\footnotesize,
  height=5.8cm
]
\ttfamily\scriptsize\raggedright
\textbf{\#\# Purpose}\\
\emph{Solve competition-style math}\\[1.5pt]

\textbf{\#\# Approach}\\
3. Notation for all unknowns\ldots\\
4. Solve step by step\ldots\\[1.5pt]

\dup{\textbf{\#\# Rules}}
\hfill{\normalfont\tiny\textcolor{LossyGrey}{$\times16$ bundle locations}}\\
\dup{- Never round intermediate\ldots}\\
\dup{- Reduce fractions\ldots}\\[1.5pt]

\dup{\textbf{\#\# Output}}
\hfill{\normalfont\tiny\textcolor{LossyGrey}{$\times16$ bundle locations}}\\
\dup{- \textbackslash boxed\{...\} final answer}\\
\dup{- multi-answer: comma list}
\hfill{\normalfont\tiny\textcolor{LossyGrey}{$\times15$}}\\
\dup{- single: output only that}
\hfill{\normalfont\tiny\textcolor{LossyGrey}{$\times7$}}\\[1.5pt]

\textbf{\#\# Verification}\\
- Substitute back\\
- Sanity-check units
\end{tcolorbox}
\end{minipage}\hfill
\begin{minipage}[t]{0.33\textwidth}
\begin{tcolorbox}[
  enhanced,
  colback=SkillZipPaleGrey,
  colframe=LossyGrey,
  boxrule=0.6pt,
  arc=1.5pt,
  title={\bfseries\sffamily\small Flat-concat \base\ \ding{55}\ (branch removed)},
  left=3pt,
  right=3pt,
  top=2pt,
  bottom=2pt,
  fonttitle=\footnotesize,
  coltitle=LossyGrey,
  height=5.8cm
]
\scriptsize\raggedright
\gone{\textbf{13 of 15 strategy branches removed.}}\\
\gone{Only \texttt{round\_00} and \texttt{round\_10} remain.}\\[2pt]

\gone{Root reduced to a flat routing list:}\\
\gone{\ttfamily~- {-}{-}{-}name: \ldots}\\
\gone{\ttfamily~- read [round\_05](\ldots)}\\
\gone{\ttfamily~- read [specialist](sub/\ldots)}\\[4pt]

\normalfont\scriptsize
\renewcommand{\arraystretch}{1.15}
\begin{tabular}{@{}c@{\,}l@{}}
\textcolor{LossyGrey}{\ding{55}} & \gone{Rules removed}\\
\textcolor{LossyGrey}{\ding{55}} & \gone{Output instructions removed}\\
\textcolor{LossyGrey}{\ding{55}} & \gone{Verification steps removed}\\
\textcolor{LossyGrey}{\ding{55}} & \gone{Routing link removed}\\
\textcolor{LossyGrey}{\ding{55}} & \gone{Branch-specific approach removed}
\end{tabular}\\[3pt]

\normalfont\scriptsize
Branch content retained: \caveat{0.2\%};
routing path: \caveat{unavailable}.
\end{tcolorbox}
\end{minipage}\hfill
\begin{minipage}[t]{0.33\textwidth}
\begin{tcolorbox}[
  enhanced,
  colback=white,
  colframe=PreprintBlue,
  boxrule=1pt,
  arc=1.5pt,
  title={\bfseries\sffamily\small \method\ \ding{51}\ (70\% fewer tokens)},
  left=3pt,
  right=3pt,
  top=2pt,
  bottom=2pt,
  fonttitle=\footnotesize,
  height=5.8cm
]
\scriptsize\raggedright
\ttfamily
\textbf{\#\# Workflow}\\
1. Read \kept{[rounding rule]}
   \,{\tiny\textcolor{PreprintGreen}{$\times16$ locations}}\\
~~\kept{[multi-answer]}
   \,{\tiny\textcolor{PreprintGreen}{$\times15$}}
   \kept{[single]}
   \,{\tiny\textcolor{PreprintGreen}{$\times7$}}\\
2. Notation for unknowns\ldots\\
3. Ordered vs.\ unordered\ldots\\
4. \guard{If ambiguous, state assumptions}\\
5. \textbackslash boxed\{...\}\\

\textbf{\#\# Verification}\\
- Substitute back (branch-specific)\\[3pt]

\normalfont\scriptsize
\renewcommand{\arraystretch}{1.15}
\begin{tabular}{@{}c@{\,}l@{}}
\textcolor{PreprintGreen}{\ding{51}}
  & Rules moved to a \kept{linked shared module}\\
\textcolor{PreprintGreen}{\ding{51}}
  & Output instructions moved to a \kept{linked shared module}\\
\textcolor{PreprintGreen}{\ding{51}}
  & Verification steps preserved verbatim\\
\textcolor{PreprintGreen}{\ding{51}}
  & Routing link remains \win{intact}\\
\textcolor{PreprintGreen}{\ding{51}}
  & Branch-specific approach preserved
\end{tabular}
\end{tcolorbox}
\end{minipage}

\vspace{3pt}

\noindent
\fcolorbox{Line}{PaleGray}{
\parbox{0.97\textwidth}{
\scriptsize
\textbf{Legend.}~
\colorbox{DupYellow}{Yellow} indicates content repeated across the bundle.
\colorbox{KeepGreen}{Green} indicates content stored once in a linked shared module.
\colorbox{GuardBlue}{Blue} indicates a guarded instruction preserved verbatim.
\gone{Grey} indicates content removed by Flat-concat \base.
}}\\[3pt]

\begin{minipage}[t]{0.38\textwidth}
\begin{tcolorbox}[
  enhanced,
  colback=KeepGreen!30,
  colframe=PreprintGreen,
  boxrule=0.6pt,
  arc=1.5pt,
  title={\bfseries\sffamily\scriptsize Shared module},
  left=3pt,
  right=3pt,
  top=2pt,
  bottom=2pt,
  fonttitle=\scriptsize
]
\ttfamily\tiny\raggedright
\#\# Notes\\
- Never round intermediate results unless the\\
~~problem explicitly asks for a decimal.\\
- Reduce fractions to lowest terms and\\
~~rationalize denominators where standard.\\
- For counting: ordered vs.\ unordered,\\
~~with vs.\ without replacement.\\
- Give the answer in simplest exact form.\\[2pt]

\normalfont\tiny
\emph{This 315-token module is stored once and referenced from
16 locations across the bundle. Replacing the other 15 copies avoids
$15 \times 315 = 4{,}725$ repeated tokens.}
\end{tcolorbox}
\end{minipage}\hfill
\begin{minipage}[t]{0.30\textwidth}
\centering
\includegraphics[width=\linewidth]{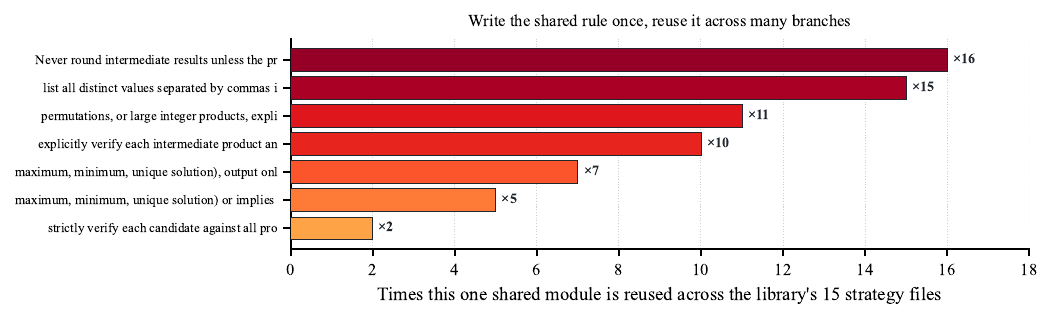}
\end{minipage}\hfill
\begin{minipage}[b]{0.28\textwidth}
\begin{tcolorbox}[
  enhanced,
  colback=SkillZipPaleBlue,
  colframe=PreprintBlue,
  boxrule=0.6pt,
  arc=2pt,
  title={\bfseries\sffamily\scriptsize Case summary},
  left=3pt,
  right=3pt,
  top=2pt,
  bottom=2pt,
  fonttitle=\scriptsize
]
\scriptsize
\begin{tabular}{@{}lr@{}}
Original branch & 527 tok\\
\rowcolor{SkillZipPaleBlue}
\method{} branch & \win{160 tok}\\
Token reduction & \hl{70\%}\\
Shared modules & 7\\
Branches retained & \win{15/15}\\
Flat-concat retained & \caveat{2/15}
\end{tabular}
\end{tcolorbox}
\end{minipage}

\caption{\textbf{One real strategy file (\texttt{round\_02}, \texttt{qwen3.6-plus} library) under three methods, plus the reuse it enables.} \colorbox{DupYellow}{Yellow} marks text duplicated across branches (with its repeat count); \colorbox{KeepGreen}{green} marks the shared module \method{} factors it into; \colorbox{GuardBlue}{blue} marks a guarded rule kept word for word. \emph{Bottom left:} how many of the 15 strategy files reuse each shared module \method{} created. \emph{Bottom right:} the token and reachability outcome for this file. Flat-concat's big ratio comes from \gone{deleting the file}; \method{}'s saving comes from writing each shared rule \emph{once}.}
\label{fig:example}
\end{figure*}

Figure~\ref{fig:example} illustrates the source of the measured savings.
Flat-concat \base retains only two of the fifteen strategy branches, making
\texttt{round\_02} unavailable to the agent. In contrast, \method reduces
this file from 527 to 160 tokens while keeping all fifteen branches
reachable. Repeated rules are stored once and referenced from each relevant
location, whereas the branch-specific workflow, verification steps, and
guarded instructions remain unchanged. The reduction therefore comes from
cross-file reuse rather than the removal of an execution path.

\subsection{How Do Savings Accumulate Across Repeated Use?}

Because a skill is deployed once but may be invoked many times, deployment size alone does not determine its total cost. We therefore measure cumulative token cost as the sum of the one-time deployment cost and the runtime cost of executing $N$ tasks. Table~\ref{tab:workload} reports the percentage reduction relative to the Evolved Bundle. We separately identify methods that retain at least $95\%$ of the original lines, since reductions obtained by removing substantial skill content are not directly comparable.

\begin{table}[!t]
\centering
\caption{Reduction in cumulative token cost relative to the Evolved Bundle, averaged across benchmarks. The cost includes one deployment and $N$ task executions. The retention criterion requires a method to preserve at least $95\%$ of the original lines; $\dagger$ denotes methods that do not satisfy this criterion.}
\label{tab:workload}
\resizebox{\columnwidth}{!}{%
\begin{tabular}{lcccccc}
\toprule
Method & Keeps skill & $N{=}1$~$\uparrow$ & $N{=}10$~$\uparrow$ & $N{=}50$~$\uparrow$ & $N{=}200$~$\uparrow$ & $N{=}1000$~$\uparrow$ \\
\midrule
Root-only \base & no$^{\dagger}$ & +5.1\% & +11.5\% & +13.8\% & +14.3\% & +14.5\% \\
Flat-concat \base & no$^{\dagger}$ & +63.6\% & +63.1\% & +62.9\% & +62.9\% & +62.9\% \\
SkillReducer & yes & +5.1\% & +5.4\% & +5.4\% & +5.5\% & +5.5\% \\
Expert Progressive & yes & +13.1\% & +11.7\% & +11.3\% & +11.1\% & +11.1\% \\
\rowcolor{SkillZipPaleBlue}\method~(One-Shot) & yes & +17.9\% & +13.9\% & +12.5\% & +12.2\% & +12.1\% \\
\bottomrule
\end{tabular}}
\end{table}

Among the methods that satisfy the retention criterion, \method achieves the largest cumulative saving at every workload size. Its reduction is $17.9\%$ for a single task and remains $12.1\%$ after 1,000 tasks, compared with $11.1\%$ for Expert Progressive and $5.5\%$ for SkillReducer. As $N$ increases, the effect of the one-time deployment cost diminishes and the results approach each method's per-run saving. The ordering of the three eligible methods nevertheless remains unchanged. Root-only and Flat-concat \base sometimes report larger reductions, but both fall below the retention threshold and are therefore excluded from this comparison. Figure~\ref{fig:workload}(a) presents the same results as cumulative-cost curves.


\begin{takeawaybox}
Runtime cost matters more than a one-time shipping ratio. Among methods that preserve the skill, \method is the \win{cheapest} at every workload size.
\end{takeawaybox}

\subsection{What Happens as the Library Keeps Growing}

Self-evolving libraries do not stay small. We grow one step by step, adding branches that repeat the same shared text, and record both sides of the ledger at each size.

\begin{table}[!t]
\centering
\caption{Behavior at the smallest and largest tested libraries, averaged over benchmarks. Retention and routing use the largest size.}
\label{tab:scaling}
\resizebox{\columnwidth}{!}{%
\begin{tabular}{lcccccc}
\toprule
& \multicolumn{2}{c}{Shipped saving} & \multicolumn{2}{c}{Per-run saving} & \multicolumn{2}{c}{At largest size} \\
\cmidrule(lr){2-3}\cmidrule(lr){4-5}\cmidrule(lr){6-7}
Method & small~$\uparrow$ & large~$\uparrow$ & small~$\uparrow$ & large~$\uparrow$ & Kept~$\uparrow$ & Routing~$\uparrow$ \\
\midrule
Root-only \base & +2.0\% & +5.2\% & +8.7\% & +3.5\% & 0.896 & 0.000 \\
SkillReducer & +5.4\% & +4.3\% & +4.7\% & +2.1\% & 1.000 & 1.000 \\
Expert Progressive & +12.4\% & +22.6\% & +18.4\% & +10.4\% & 1.000 & 1.000 \\
\rowcolor{SkillZipPaleBlue}\method~(One-Shot) & +21.8\% & +25.5\% & +17.6\% & +5.3\% & 0.984 & 1.000 \\
\bottomrule
\end{tabular}}
\end{table}

\method is the only approach that preserves the complete skill and routing table at every size, and its per-run cost remains flat as the library grows. At the largest size, however, Expert Progressive removes more shipped bytes by moving all repeated text into one root-linked file. Every run then reads that growing file, so its runtime saving declines. \method keeps each shared block within the branches that use it, trading some storage saving for lower execution cost. Routing also has an unavoidable price: the root retains one condition per branch. An attempted on-demand grouping added a hop to every run and failed the never-inflate check. Figure~\ref{fig:workload}(b,c) shows this trade-off.

\begin{figure}[!t]
\centering
\includegraphics[width=\columnwidth]{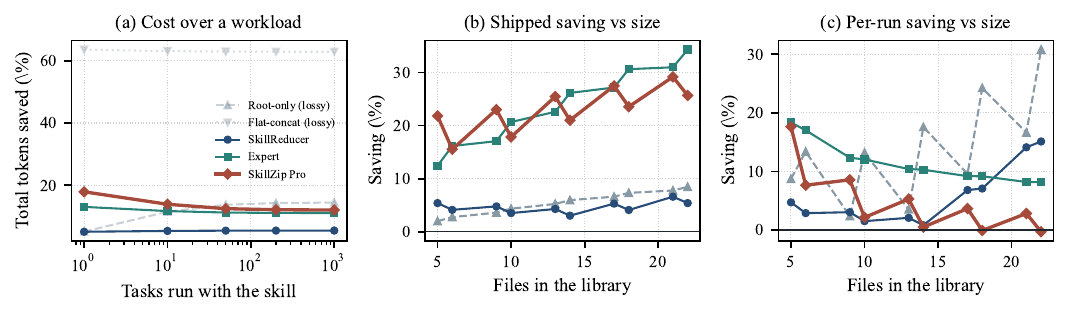}
\caption{(a) Cumulative token saving as tasks grow; faded methods lost skill content. (b) Shipped and (c) per-run saving as library size grows. \method keeps per-run saving stable.}
\label{fig:workload}
\end{figure}


\subsection{Real Evolved Skill Libraries}
\label{sec:evolved}

The bundles used so far repeat about a third of their text. Libraries that a self-evolving agent actually leaves behind repeat far more, because every round appends its output rules, its checklist and its growing list of past mistakes into whichever skill it edits. We therefore take the evolution runs stored with this project, keep \emph{every} round as its own file, and assemble each run into one progressively loaded library: \textbf{17 files, 15 rounds, about 20{,}000 tokens, and 79--84\% repeated text}. This is the regime the method is built for, so it carries the headline numbers.

\begin{table}[!t]
\centering
\caption{Real evolved libraries, averaged over three runs. $\dagger$ marks incomplete bundles. For \method, brackets count published compressions; otherwise the never-inflate check republishes the source.}
\label{tab:evolved}
\resizebox{\columnwidth}{!}{%
\begin{tabular}{lccccc}
\toprule
Method & Kept~$\uparrow$ & Shipped~$\uparrow$ & Always loaded~$\uparrow$ & One run~$\uparrow$ & Routing~$\uparrow$ \\
\midrule
Evolved library (uncompressed) & 1.000 & +0.0\% & +0.0\% & +0.0\% & 1.000 \\
Root-only \base & 0.951 & +1.6\% & +69.2\% & +12.5\% & 0.000 \\
\rowcolor{SkillZipPaleGrey}\color{LossyGrey}Flat-concat \base\,\xmark & 0.002 & +64.4\% & +82.4\% & +90.2\% & 0.000 \\
SkillReducer & 0.954 & +6.3\% & +63.0\% & +30.9\% & 0.125 \\
Expert Progressive & 1.000 & +25.0\% & -2.9\% & +26.0\% & 1.000 \\
\rowcolor{SkillZipPaleBlue}\method~(One-Shot) & 1.000 & +23.1\% (2/3) & +1.7\% & +17.7\% & 1.000 \\
\bottomrule
\end{tabular}}
\end{table}

Three readings matter. First, where it publishes a result, \method removes \hl{34.7\%} of shipped tokens versus \emph{25.0\%} for the strongest faithful baseline, while preserving the skill and routing list. Second, that baseline moves all repeated text into one root-linked file, making the always-loaded layer \emph{longer} ($-2.9\%$); only \method improves every layer at full fidelity. Third, root-only and SkillReducer gain $+69\%$ and $+63\%$ in the root by rewriting away routing, reflected in routing scores of $0.000$ and $0.125$.

On the third library, \method \emph{declined to compress} and republished the source byte for byte. The candidate passed the audit but did not improve the objective, so Algorithm~\ref{alg:skillzippro-oneshot} rejected it. This fallback is intentional: the method never ships an unproven regression. Continual mode later compresses the same library to $+49.9\%$ by repacking at a more favorable point in the stream (Section~\ref{sec:evolvedcontinual}).

\begin{figure}[!t]
\centering
\includegraphics[width=\columnwidth]{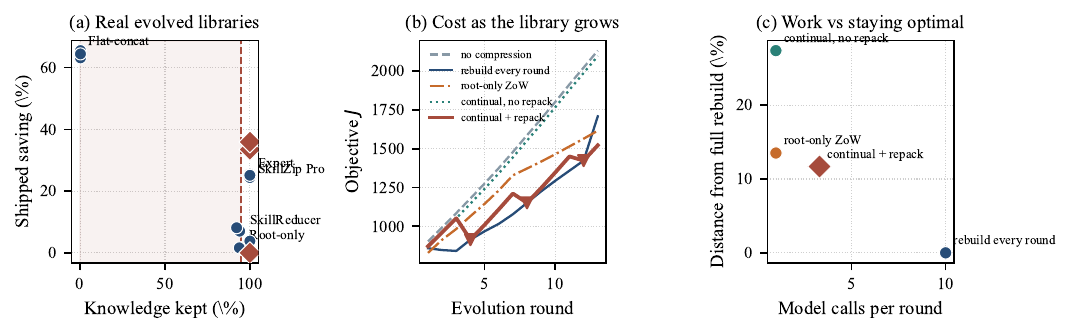}
\caption{(a) Saving versus retained knowledge on evolved libraries. (b) Published cost over rounds; triangles mark repacks. (c) Update work versus final drift from One-Shot rebuilding.}
\label{fig:evolved}
\end{figure}

\begin{takeawaybox}
On real self-evolved libraries, \method removes \win{about one third} of shipped tokens where it publishes---roughly ten points more than the strongest content-preserving baseline. It is also the only method that improves shipped, always-loaded, and per-run cost together; otherwise it keeps the original bundle.
\end{takeawaybox}

\subsection{Continual Compression Over a Long Evolution Stream}
\label{sec:evolvedcontinual}

Continual compression is intended to incorporate new skill updates without rebuilding the entire library after every change. We evaluate whether it can reduce this update cost while still recovering redundancy that accumulates across multiple rounds. Specifically, we replay \textbf{13 real evolution rounds} for each of three libraries and compare five compression schedules. All schedules begin from the same initial state and receive the same sequence of updates.

\begin{table*}[!t]
\centering
\caption{Continual compression over 13 real evolution rounds, averaged across three libraries. Savings are measured relative to publishing every update without compression. Drift measures the difference from rebuilding the complete library at the same round; a negative value indicates a smaller result than the rebuild. Compressor calls, processed bytes, and wall-clock time are reported per round.}
\label{tab:evolvedcont}
\resizebox{0.92\textwidth}{!}{%
\begin{tabular}{lrrrrrr}
\toprule
Schedule &
Shipped saving~$\uparrow$ &
Per-run saving~$\uparrow$ &
Drift~$\downarrow$ &
Calls/round~$\downarrow$ &
Bytes/round~$\downarrow$ &
Seconds/round~$\downarrow$ \\
\midrule
Publish without compressing
& +0.0\% & +0.0\% & +0.000 & 0.00 & 3201 & 0.000 \\

Rebuild from scratch every round
& +35.1\% & +18.7\% & +0.000 & 10.00 & 13388 & 0.234 \\

\base~Zip-on-Write (root only)
& +2.4\% & +19.6\% & -0.027 & 1.00 & 389 & 0.094 \\

\method~Continual, no repack
& +2.7\% & +0.8\% & +0.262 & 1.00 & 3201 & 0.008 \\

\rowcolor{SkillZipPaleBlue}
\method~Continual + repack
& +48.1\% & +28.2\% & -0.083 & 3.31 & 5697 & 0.068 \\
\bottomrule
\end{tabular}}
\end{table*}

    \textbf{\method Continual with repacking removes \hl{48.1\%} of shipped tokens and \hl{28.2\%} of per-run tokens while requiring only \hl{3.31} compressor calls per round.} Rebuilding the complete library after every update requires \hl{10} calls per round, so Continual mode uses approximately one third as many calls. It also reduces the average processing time from $0.234$ to $0.068$ seconds per round. The result is slightly smaller than the corresponding rebuild, as indicated by the drift of $-0.083$. This difference occurs because the final rebuild rejects a non-improving candidate, whereas Continual mode retains the valid compressed state produced by an earlier repack.

Periodic repacking is necessary to recover redundancy introduced across different evolution rounds. Without repacking, the shipped-token saving falls to \emph{2.7\%}, the per-run saving falls to \emph{0.8\%}, and drift increases to $+0.262$. This result shows that processing each update locally is insufficient when related content is added in separate rounds. Root-only Zip-on-Write requires only one call and 389 processed bytes per round, but it reduces the shipped library by just \emph{2.4\%} because it cannot consolidate repeated content across files. Although it reduces the root-level per-run cost, it leaves most bundle-wide redundancy unchanged. Figure~\ref{fig:evolved}(b,c) shows how these differences accumulate over the evolution stream.

\begin{takeawaybox}
With periodic repacking, Continual mode recovers cross-round redundancy while using \win{approximately one third} as many compressor calls as rebuilding after every update. Without repacking, shipped-token savings remain at only \caveat{2.7\%}.
\end{takeawaybox}

\subsection{When Should Compression Be Switched On?}
\label{sec:startearly}

A self-evolving library is republished and reloaded after every round. Waiting therefore costs more than final size: the agent carries the uncompressed library through every earlier round. We switch compression on at $k\in\{1,4,7,10,13\}$, plus a never-compress baseline.

The metric that answers the question is \textbf{the total tokens the agent carries across the whole stream}: for each round we add up the full size of the library published at that round. We also report the total compressor calls, so a schedule cannot look cheap merely by doing less work.

\begin{table}[htbp]
\centering
\caption{When to switch compression on, averaged over the three libraries and the whole 15-round stream. Carried tokens is the sum, over rounds, of the size of the library the agent holds that round.}
\label{tab:start}
\resizebox{\columnwidth}{!}{%
\begin{tabular}{lrrr}
\toprule
Compression is on & Tokens carried~$\downarrow$ & Tokens saved~$\uparrow$ & Calls~$\downarrow$ \\
\midrule
\rowcolor{SkillZipPaleBlue}from round 1 & 48055 & +40.8\% & 43 \\
from round 4 & 49627 & +38.9\% & 36 \\
from round 7 & 53234 & +34.4\% & 42 \\
from round 10 & 56179 & +30.8\% & 30 \\
from round 13 & 66139 & +18.5\% & 16 \\
never switch it on & 81176 & +0.0\% & 0 \\
\bottomrule
\end{tabular}}
\end{table}

\textbf{The ordering is monotone and the effect is large.} Switching on at the first round saves \hl{40.8\%} of everything the agent carries over the stream; waiting until round 13 recovers only \caveat{18.5\%}, less than half as much, and starting at round 1 beats starting at round 7 by a clear \hl{9.7\%}. Every extra round of waiting leaves saving on the table.

The reason is worth stating precisely, because it is not simply ``smaller is better''. Repetition in an evolved library is \emph{cumulative}: round $t$ re-appends the same output rules and checklist that rounds $1..t-1$ already contain, so the number of duplicate copies grows with $t$. Compression removes copies, not rounds. Starting at round $k$ therefore leaves the agent paying full price on rounds $1..k-1$, and those tokens can never be recovered afterwards -- the final library can be compressed just as well later, but the intervening runs are already spent. Figure~\ref{fig:start}(b) plots the running total the agent has paid: the three curves never cross, so an earlier start is cheaper at \emph{every} round and the gap only widens. Early compression is strictly cheaper because delayed rounds cannot be recovered. Starting at round 1 requires 43 compressor calls over 15 rounds, versus 16 for a late start, and uses no task rollouts.

\begin{figure}[!t]
\centering
\includegraphics[width=\columnwidth]{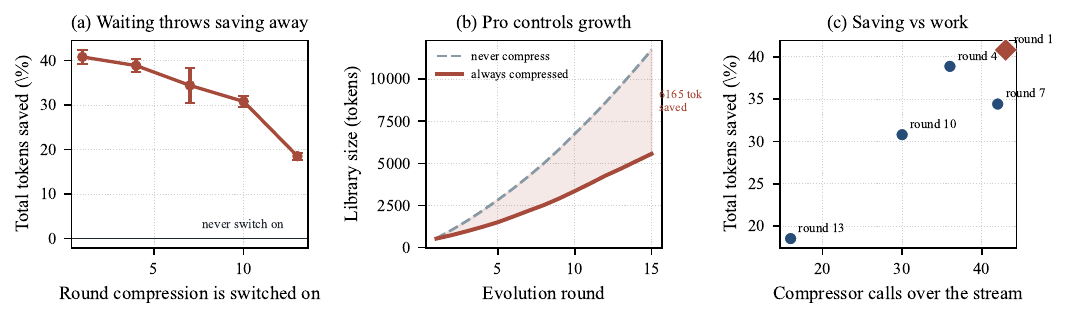}
\caption{(a) Total tokens saved over the stream against the round compression is switched on; bars show the spread across the three libraries. (b) Actual library size at each round: uncompressed grows linearly while Pro keeps it roughly \hl{50\%} smaller by removing repeated text as it appears. (c) The same total saving plotted against the number of compressor calls.}
\label{fig:start}
\end{figure}


\begin{insightbox}{2}
For a self-evolving agent, compression should start early. Repetition accumulates each round, and later compression cannot recover context already paid for. The lowest cumulative cost therefore comes from enabling compression in the first evolution round.
\end{insightbox}

\subsection{Why the Savings Compound: A Look Inside Self-Evolution}
\label{sec:compounding}

The results so far show \emph{that} \method{} helps; this subsection shows \emph{why}, and the reason is the surprising finding in the paper. We measured, at every evolution round, how much of the growing library is genuinely new and how much simply repeats text the agent has already written, by compressing each round-$k$ prefix and comparing its size against the raw prefix. Three facts fall out, all from the same real evolved libraries and all consistent across runs.

\begin{figure*}[!t]
\centering
\includegraphics[width=0.92\textwidth]{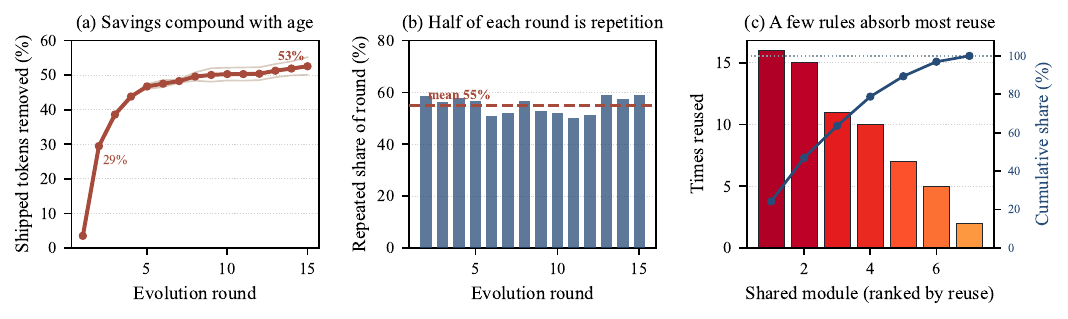}
\caption{\textbf{Why the savings compound.} (a) The share of shipped tokens \method{} removes climbs from \hl{29\%} at round~2 to \hl{53\%} at round~15 and is still rising -- the method gets \emph{more} valuable the longer the agent evolves. (b) The reason: the fraction of each new round that merely repeats earlier text is stable at \win{55\%$\pm$3\%} -- more than half of everything a self-evolving agent writes, it has written before. (c) That repetition is concentrated: a few ``universal'' rules absorb most of the reuse (top two of seven modules $=$ \hl{47\%}), so a small shared library covers the bulk of the redundancy.}
\label{fig:compounding}
\end{figure*}

\textbf{Finding 1 --- repetition accumulates during self-evolution.}
Across rounds, \win{$55\%\pm3\%$} of the content added in each round overlaps with existing text (Figure~\ref{fig:compounding}b). This proportion remains relatively stable across rounds and models, as append-only updates often repeat output formats, verification checklists, and previously accumulated rules when extending a skill.

\textbf{Finding 2 --- so the savings compound with age.} Because redundancy accumulates, the share \method{} can remove \emph{grows} with the library: from \hl{29\%} at round~2 to \hl{53\%} at round~15, still climbing (Figure~\ref{fig:compounding}a). \textbf{\emph{The method does not have diminishing returns; it has increasing ones.}} The longer an agent runs, the more of its library is repetition, and the more \method{} saves -- the opposite of how one-shot compressors of static prompts behave.

\textbf{Finding 3 --- the repetition is concentrated.} The reuse is heavy-tailed: of the seven shared modules \method{} factors out on the worked library, the top two account for \hl{47\%} and the top three for \hl{64\%} of all reuse (Figure~\ref{fig:compounding}c). A handful of ``universal'' rules -- how to format the answer, when to keep exact values -- are re-derived over and over, so a very small shared library covers most of the redundancy.

\textbf{\emph{Remark.}} 
For references or subskills that are used only through the root skill, these universal rules are redundant and can be removed through \emph{Persistent Compression}. If a reference or subskill must also remain independently usable, however, a rule repeated from the root may still be required by its standalone entry contract and cannot be removed solely because of that overlap. In this case, \emph{Transient Compression} can eliminate the duplication within a task-specific execution view while leaving the original resource unchanged. We will discuss this distinction in detail in Section~\ref{sec:lifecycleexp}.

\begin{insightbox}{3}
\textbf{Self-evolution manufactures redundancy, and \method{} turns that into a compounding advantage.} More than half of every new round (\win{55\%}) is text the agent has already written, so the share \method{} removes \emph{rises} with age -- \hl{29\%}$\to$\hl{53\%} over fifteen rounds and still climbing. A compressor for evolving skills is therefore worth \emph{more} the longer it runs, not less; and because the repetition concentrates in a few universal rules, a tiny shared library captures most of it. This is the reason to put compression \emph{inside} the evolution loop rather than treating it as occasional cleanup.
\end{insightbox}

\subsection{Runtime Overhead of Compression}

Table~\ref{tab:efficiency} measures the computational cost of producing a compressed bundle. We report wall-clock time, model calls, model tokens, agent rollouts, and peak memory. These measurements isolate compression overhead from the subsequent cost of executing tasks with the compressed skill.

\begin{table}[!t]
\centering
\caption{Mean cost of compressing one skill bundle. Structural extraction, cross-file transformation, and post-compression auditing are deterministic and require no model calls, model tokens, or agent rollouts.}
\label{tab:efficiency}
\resizebox{\columnwidth}{!}{%
\begin{tabular}{lrrrrr}
\toprule
Method &
Time (s)~$\downarrow$ &
Calls~$\downarrow$ &
Tokens~$\downarrow$ &
Rollouts~$\downarrow$ &
Peak GB~$\downarrow$ \\
\midrule
Root-only \base
& 0.11 & 0 & 0 & 0 & 0.02 \\

Flat-concat \base
& 0.14 & 0 & 0 & 0 & 0.02 \\

SkillReducer
& 0.12 & 0 & 0 & 0 & 0.02 \\

\rowcolor{SkillZipPaleBlue}
\method~(One-Shot)
& 0.21 & 0 & 0 & 0 & 0.02 \\
\bottomrule
\end{tabular}}
\end{table}

\method compresses a bundle in \textbf{0.21 seconds on average}, compared with \textbf{0.11 seconds} for Root-only \base. The additional $0.10$ seconds are used to resolve the resource graph, identify content that can be shared across files, construct conditional capsules, and audit the resulting bundle from disk. Despite these additional checks, \method requires \win{no task feedback} and \win{no agent rollouts}. Its peak memory usage remains \textbf{0.02 GB}, identical to the other methods.





\subsection{Does a Compressed Bundle Transfer Across Models?}

A compressed bundle should not depend on the model that later executes it. We therefore evaluate the \emph{same} \method output with three executor models(qwen3.7-max as the evolving model), without recompiling or modifying any files. For each executor, we compare task success and progressive-loading behavior against the same uncompressed Evolved Bundle. The reported results are run on the structured track of two of our three benchmarks—LiveMathematicianBench (27 held-out tasks) and BFCL-v4 (17 held-out tasks), 44 tasks in total—with each cell macro-averaged over the two benchmarks.

\begin{table}[!t]
\centering
\caption{Transfer of one compressed bundle across three executor models. Each executor runs the same \method bundle and the corresponding uncompressed Evolved Bundle on identical tasks. Required-resource recall and irrelevant loading measure whether compression changes the resources selected during execution.}
\label{tab:transfer}
\resizebox{\columnwidth}{!}{%
\begin{tabular}{lcccccc}
\toprule
& \multicolumn{2}{c}{Task success}
& \multicolumn{2}{c}{Req. recall}
& \multicolumn{2}{c}{Irrel. load} \\
\cmidrule(lr){2-3}
\cmidrule(lr){4-5}
\cmidrule(lr){6-7}
Executor
& Evolved~$\uparrow$
& \method~$\uparrow$
& Evolved~$\uparrow$
& \method~$\uparrow$
& Evolved~$\downarrow$
& \method~$\downarrow$ \\
\midrule
qwen3.7-max
& 0.541 & 0.619
& 0.679 & 0.690
& 0.321 & 0.300 \\

qwen3.6-plus
& 0.608 & 0.656
& 0.738 & 0.727
& 0.224 & 0.223 \\

kimi-k2.6
& 0.611 & 0.574
& 0.435 & 0.490
& 0.077 & 0.287 \\
\bottomrule
\end{tabular}}
\end{table}

The compressed bundle remains executable without modification under all three models, confirming that its files, relative links, and routing structure do not require an executor-specific integration. With qwen3.7-max and qwen3.6-plus, \method improves task success from $0.541$ to \win{0.619} and from $0.608$ to \win{0.656}, respectively, while required-resource recall and irrelevant loading remain close to those of the uncompressed bundle. The result is less consistent for kimi-k2.6: required recall improves from $0.435$ to \win{0.490}, but task success decreases from $0.611$ to \caveat{0.574} and irrelevant loading increases from $0.077$ to \caveat{0.287}. Thus, the bundle is structurally portable, although its loading efficiency still depends on how the executor interprets the shared routing instructions.


\begin{insightbox}{4}
\textbf{Compression changes not only bundle size, but also cross-model transferability.}
Moving repeated instructions from inline copies into shared files makes execution depend more strongly on routing. This restructuring \win{improves task success for both Qwen executors}, but causes \caveat{more irrelevant loading and slightly lower success for kimi-k2.6}. 
\end{insightbox}





\subsection{Which Part Does the Work?}

Table~\ref{tab:ablation} switches off one piece at a time, grouped by the two pillars. The first four rows ablate \textbf{\textcolor{SkillZipBlue}{Pillar~1}} (what and where to compress); the last three probe \textbf{\textcolor{PreprintGreen}{Pillar~2}} (keeping the routing intact). ``No resource graph'' approximates compressing each file on its own; ``global sharing'' deliberately ignores which branches need a block and tests the failure predicted by Proposition~\ref{prop:sparseroot}.

\begin{table*}[!t]
\centering
\caption{Turning off one piece at a time (grown version). ``Damage'' counts files left unreachable plus broken links. Quality columns use held-out runs and are reported for the full method.}
\label{tab:ablation}
\resizebox{0.93\textwidth}{!}{%
\begin{tabular}{lccccccc}
\toprule
Variant & $J$ saving~$\uparrow$ & Shipped~$\uparrow$ & Always loaded~$\uparrow$ & Per run~$\uparrow$ & Task success~$\uparrow$ & Req.\ recall~$\uparrow$ & Damage~$\downarrow$ \\
\midrule
\rowcolor{SkillZipPaleBlue}\method~(One-Shot) & 0.142 & +21.1\% & +7.3\% & +15.8\% & 0.560 & 0.755 & 0 \\
$-$ Resource graph & 0.118 & +13.7\% & +7.3\% & +13.6\% & -- & -- & 0 \\
$-$ Host entailment & 0.077 & +15.3\% & +2.0\% & +8.3\% & -- & -- & 0 \\
$-$ Scoped sharing & 0.132 & +7.9\% & +7.3\% & +17.0\% & -- & -- & 0 \\
$-$ Conditional capsules & 0.126 & +26.9\% & +7.3\% & +11.8\% & -- & -- & 0 \\
$-$ Cross-file audit & 0.142 & +21.1\% & +7.3\% & +15.8\% & -- & -- & 0 \\
$-$ Routing-table lock & 0.156 & +22.2\% & +7.1\% & +17.8\% & -- & -- & 0 \\
Global sharing (unsafe control) & -0.166 & +17.0\% & -84.9\% & -0.9\% & -- & -- & 1 \\
\bottomrule
\end{tabular}}
\end{table*}

The ablations separate the pillars. \textbf{\textcolor{SkillZipBlue}{Pillar~1}} supplies the savings: removing scoped sharing loses most deployment reduction, while removing capsules trades lower storage for higher per-run cost, as Eq.~\eqref{eq:capsulethreshold} predicts. \textbf{\textcolor{PreprintGreen}{Pillar~2}} makes those savings deployable. Global sharing pushes content into the root, drives $J$ below the source ($-0.166$), and makes one file unreachable. Disabling only the routing lock leaves routing intact because reference-line preservation and the final reachability audit remain. Routing safety therefore comes from layered checks, not one switch.

We sweep $\lambda\in\{0,0.01,0.05,0.1,0.25,1\}$ and rerun the optimizer. Figure~\ref{fig:pareto} shows identical transformations and costs across the range. Each bundle-level move is accepted only when package saving exceeds added loading cost, so the decision remains stable from $\lambda=0$ to $1$. The default $\lambda=0.05$ is not tuned to the results.

\begin{figure}[!t]
\centering
\includegraphics[width=\columnwidth]{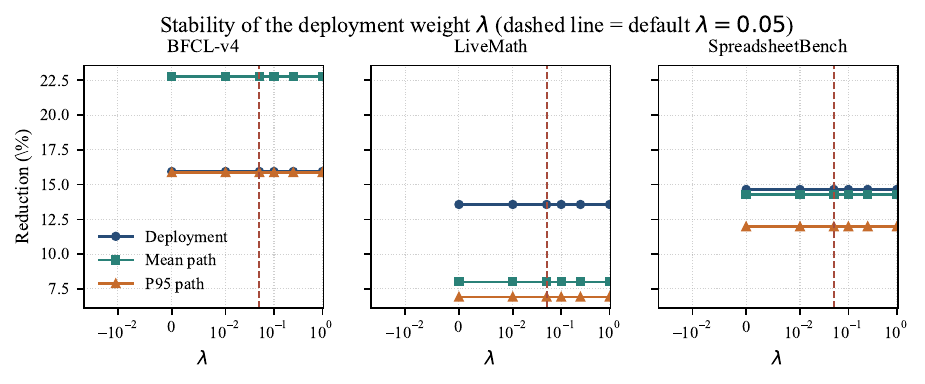}
\caption{Sensitivity to storage weight $\lambda$. Costs remain unchanged across two orders of magnitude; the dashed default is not result-tuned.}
\label{fig:pareto}
\end{figure}

\begin{takeawaybox}
Cross-file sharing produces most of the shipped saving, while capsules trade a small storage cost for cheaper runs. Ignoring branch scope breaks the bundle, and varying $\lambda$ over two orders of magnitude does not change the result.
\end{takeawaybox}

\subsection{One-Shot Versus Continual}

Patches are replayed in order into a fixed source, so every checkpoint has one identical uncompressed bundle that all schedules start from. We compare: publish without compressing; rebuild fully after every patch (the reference); root-only Zip-on-Write; continual with local repair and no repack; and continual with the preset repack rule. Continual runs may only reuse state from before the current patch.

\begin{table*}[!t]
\centering
\caption{After the last patch, averaged over benchmarks. Savings are against publishing without compressing. Drift is measured against a full rebuild of the same state. Update seconds, bytes touched, and calls are per patch.}
\label{tab:zow}
\resizebox{0.94\textwidth}{!}{%
\begin{tabular}{lrrrrrrr}
\toprule
Schedule & Shipped saving~$\uparrow$ & Per-run saving~$\uparrow$ & Drift~$\downarrow$ & Update s~$\downarrow$ & Bytes touched~$\downarrow$ & Calls/patch~$\downarrow$ & $\Delta$Quality \\
\midrule
Append Only & +0.0\% & +0.0\% & +4.6\% & 0.001 & 263 & 0.0 & 0 \\
\method~One-Shot after every patch & +12.4\% & +2.1\% & +0.0\% & 0.170 & 6250 & 8.8 & 0 \\
\base~Zip-on-Write (root only) & +9.4\% & +21.7\% & -27.7\% & 0.115 & 604 & 1.0 & 0 \\
\method~Continual, no repack & +7.7\% & +4.7\% & -0.8\% & 0.013 & 263 & 1.0 & 0 \\
\rowcolor{SkillZipPaleBlue}\method~Continual + triggered repack & +12.4\% & +2.1\% & +0.0\% & 0.060 & 1875 & 3.3 & 0 \\
\bottomrule
\end{tabular}}
\end{table*}

Continual mode with repacking reaches the full rebuild's objective with zero drift, about one third of the model calls per patch, and far fewer rewritten bytes. Between repacks it reads only the changed file; without repacking, unfactored repetition accumulates. Root-only Zip-on-Write reports a smaller objective only by rewriting routing text that \method locks; Table~\ref{tab:loading} shows the runtime loss.

\begin{figure}[!t]
\centering
\includegraphics[width=\columnwidth]{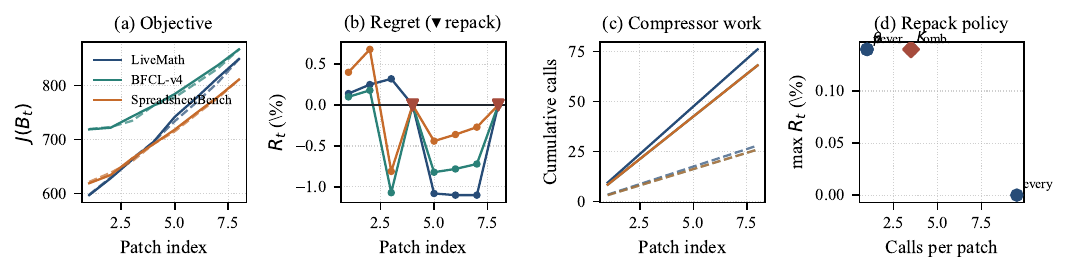}
\caption{Continual mode over patches. (a) objective per patch, full rebuild (solid) versus continual with repack (dashed); (b) drift from the full rebuild, with repack events marked; (c) model calls added up over patches; (d) repack rules, calls against worst drift.}
\label{fig:continual}
\end{figure}

\begin{table}[!t]
\centering
\caption{Repack rules on one patch stream. Repack rate is the fraction of patches rebuilt globally; worst drift spans all checkpoints.}
\label{tab:repack-sensitivity}
\resizebox{\columnwidth}{!}{%
\begin{tabular}{lrrrr}
\toprule
Rule & Repack rate & Worst drift~$\downarrow$ & Calls/patch~$\downarrow$ & Update s~$\downarrow$ \\
\midrule
Never repack & 0.00 & 0.0014 & 1.00 & 0.000 \\
Every patch & 1.00 & 0 & 9.50 & 0.413 \\
Savings only ($\theta$) & 0.00 & 0.0014 & 1.00 & 0.000 \\
Growth only ($\rho$) & 0.00 & 0.0014 & 1.00 & 0.000 \\
Patch count only ($K$) & 0.25 & 0.0014 & 3.50 & 0.119 \\
\rowcolor{SkillZipPaleBlue}Combined policy & 0.25 & 0.0014 & 3.50 & 0.106 \\
\bottomrule
\end{tabular}}
\end{table}

Table~\ref{tab:repack-sensitivity} shows the schedule matters more than any single trigger here. Rebuilding after every patch removes all drift at nearly ten times the per-patch call cost, while never rebuilding still keeps worst drift near a tenth of a percent. The saving and growth triggers never fire on this stream because drift never crosses their thresholds, so the combined rule is driven by the patch counter and lands at a quarter of the rebuild rate with the same bounded drift.

\begin{takeawaybox}
Continual mode matches a full rebuild with about one third of the model calls per patch. A patch-count trigger is sufficient at this scale; the other triggers protect faster-changing workloads.
\end{takeawaybox}

\subsection{The Compression Lifecycle: Persistent vs.\ Transient}
\label{sec:lifecycleexp}

Every result above shipped a rewritten bundle on disk. In the vocabulary of Section~\ref{sec:lifecycle} that is \emph{Persistent} compression: Tables~\ref{tab:main-cost}, \ref{tab:loading}, and \ref{tab:retention} are all persistent results, and we relabel them as such rather than re-run them. What those tables did not test is the second lifecycle---building a throwaway view per run---or the case that makes the two lifecycles diverge: a bundle whose subskills and references are called \emph{directly}, not only from the root. This subsection adds exactly that---and it is where \method's advantage is easiest to see: among all methods we test, only \method{} compresses a multi-entry bundle while keeping \emph{all} of its routing and \emph{every} declared entry independently callable, whereas the lossy baselines reach their smaller numbers only by breaking exactly those two properties.

\textbf{Setup.} We derive one controlled multi-entry bundle from the math skill: a root \texttt{SKILL.md}, a catalogued public subskill at \texttt{sub/SKILL.md}, and four conditional references carrying full task contracts. We compare Uncompressed, \base, the lossy Root-only and Flat-concat baselines, audited persistent \method, persistent \method without the multi-entry audit, and transient execution-view \method. Compression and view construction are deterministic and use no model calls; only the direct-call rollout in Table~\ref{tab:lifecycle-standalone} uses the unmodified agent. Following Section~\ref{sec:lifecyclecosts}, persistent rows report disk and per-run cost. The transient row keeps disk at $1.000$ and reports runtime saving separately.

\subsubsection{Compressing without losing a single public entry}
Table~\ref{tab:lifecycle-cost} separates methods that preserve a multi-entry bundle from those that do not. Standalone-preserving persistent \method reduces the bundle to \hl{0.884} of its bytes and per-run load from 404 to \hl{352} tokens while keeping routing, mean public independence, and worst-case independence at \win{1.000}. The transient view keeps the same guarantees without changing disk. Root-only and flat-concat report smaller numbers only after breaking routing or removing public content. Disabling the multi-entry audit exposes the failure: persistent \method reaches \hl{0.845} on disk but renames the public subskill, reducing its effective independence to \caveat{0.500}. Proposition~\ref{prop:sparseroot} therefore applies at bundle scale: every declared entry, not the desired ratio, sets the safe persistent ceiling.

\begin{table}[t]
\centering
\caption{\textbf{Persistent cost and multi-entry fidelity.} Lower cost is better; higher routing and independence are better. Blue rows are valid; grey rows break routing or independence.}
\label{tab:lifecycle-cost}
\resizebox{\columnwidth}{!}{%
\begin{tabular}{lcccccc}
\toprule
Method & Disk~$\downarrow$ & P95 ctx~$\downarrow$ & Per-run~$\downarrow$ & Route~$\uparrow$ & Pub.\ Ind.~$\uparrow$ & Worst~$\uparrow$ \\
\midrule
Uncompressed & 1.000 & 607 & 404 & 1.000 & 1.000 & 1.000 \\
\base & 0.959 & 587 & 393 & 1.000 & 1.000 & 1.000 \\
\rowcolor{SkillZipPaleGrey}\color{LossyGrey}Root-only\,\xmark & 0.972 & 537 & 242 & 0.000 & 0.625 & 0.250 \\
\rowcolor{SkillZipPaleGrey}\color{LossyGrey}Flat-concat\,\xmark & 0.379 & 190 & 74 & 0.000 & 0.062 & 0.000 \\
\rowcolor{SkillZipPaleBlue}Standalone-preserving Persistent \method & 0.884 & 607 & 352 & 1.000 & \win{1.000} & \win{1.000} \\
\rowcolor{SkillZipPaleGrey}\color{LossyGrey}Persistent \method, no audit\,\xmark & 0.845 & 565 & 331 & 1.000 & 0.500 & 0.000 \\
\rowcolor{SkillZipPaleBlue}Transient Execution-View \method & 1.000 & 607 & 404 & 1.000 & \win{1.000} & \win{1.000} \\
\bottomrule
\end{tabular}}
\end{table}

\subsubsection{Direct calls need the audit or a transient view}
Table~\ref{tab:lifecycle-independence} isolates the failure. Without the audit, persistent \method retains \hl{0.929} of public-entry content but only \caveat{0.500} discoverability because it renames the public subskill. Audited persistent and transient modes remain at \win{1.000} on all four measures: one rejects hidden entries, and the other never edits them. Direct calls need one of these protections.

\begin{table}[t]
\centering
\caption{\textbf{Why a public entry loses independence} (math bundle). ``Discoverable'' is whether a direct call still finds the entry ($\uparrow$); ``Content'' is how much of the found entry's knowledge survives ($\uparrow$); ``Effective'' is their per-entry product, averaged---the number Table~\ref{tab:lifecycle-cost} reports. ``Cond.'' is conditional-entry content. Persistent \method{} with the audit off loses capability by \emph{renaming} a public entry, not by deleting its text.}
\label{tab:lifecycle-independence}
\resizebox{\columnwidth}{!}{%
\begin{tabular}{lcccc}
\toprule
Method & Pub.\ disc.~$\uparrow$ & Pub.\ content~$\uparrow$ & Pub.\ eff.~$\uparrow$ & Cond.\ content~$\uparrow$ \\
\midrule
Uncompressed & 1.000 & 1.000 & 1.000 & 1.000 \\
\base & 1.000 & 1.000 & 1.000 & 1.000 \\
\rowcolor{SkillZipPaleGrey}\color{LossyGrey}Root-only\,\xmark & 1.000 & 0.625 & 0.625 & 1.000 \\
\rowcolor{SkillZipPaleGrey}\color{LossyGrey}Flat-concat\,\xmark & 1.000 & 0.062 & 0.062 & 0.150 \\
\rowcolor{SkillZipPaleGrey}\color{LossyGrey}Persistent \method, no audit\,\xmark & 0.500 & 0.929 & 0.500 & 0.887 \\
\rowcolor{SkillZipPaleBlue}Standalone-preserving Persistent \method & \win{1.000} & 1.000 & \win{1.000} & 1.000 \\
\rowcolor{SkillZipPaleBlue}Transient Execution-View \method & \win{1.000} & 1.000 & \win{1.000} & 1.000 \\
\bottomrule
\end{tabular}}
\end{table}

The live rollout confirms this result. Table~\ref{tab:lifecycle-standalone} calls each entry through the unmodified agent. Without the audit, persistent \method cannot start the renamed public subskill (discoverability and standalone success are both $0.000$), although unrenamed conditional references still run. Audited persistent and transient modes keep every entry discoverable; their standalone and root-mediated success matches uncompressed within this small rollout's noise ($n{=}3$ for the public entry, $n{=}8$ per reference). Discoverability is deterministic: only the audit or a transient view prevents cross-file deduplication from hiding a public entry.

\begin{table}[t]
\centering
\caption{\textbf{Direct-call rollout.} Pub. disc. tests public-subskill startup; Pub., Cond., and Root report success by entry type. Discoverability is deterministic; $n{=}3$ public and $n{=}8$ per reference.}
\label{tab:lifecycle-standalone}
\resizebox{\columnwidth}{!}{%
\begin{tabular}{lcccc}
\toprule
Method & Pub.\ disc.~$\uparrow$ & Pub.~$\uparrow$ & Cond.~$\uparrow$ & Root~$\uparrow$ \\
\midrule
Uncompressed & 1.000 & 0.333 & 0.375 & 0.364 \\
\rowcolor{SkillZipPaleGrey}\color{LossyGrey}Persistent \method, no audit\,\xmark & 0.000 & 0.000 & 0.417 & 0.394 \\
\rowcolor{SkillZipPaleBlue}Standalone-preserving Persistent \method & \win{1.000} & 0.333 & 0.375 & -- \\
\rowcolor{SkillZipPaleBlue}Transient Execution-View \method & \win{1.000} & 0.667 & 0.417 & -- \\
\bottomrule
\end{tabular}}
\end{table}

\subsubsection{Matching the four modes to update rate and call pattern}
Table~\ref{tab:lifecycle-sensitivity} reports each workload boundary. Persistent load stays below transient load at every direct-call count because compression is paid once; stable, heavily used bundles favor \emph{one-shot persistent}. A global repack costs 273\,ms versus 183\,ms for a local transient rebuild, so changing direct-call bundles favor cached \emph{continual transient} above roughly one edit per ten runs. Figure~\ref{fig:lifecycle} shows the crossover.

\begin{table*}[t]
\centering
\caption{\textbf{Lifecycle sensitivity.} Each row reports the measured boundary at which one lifecycle becomes cheaper or safer.}
\label{tab:lifecycle-sensitivity}
\resizebox{\textwidth}{!}{%
\begin{tabular}{@{}llp{0.34\textwidth}p{0.30\textwidth}@{}}
\toprule
Sweep & Range & Measured effect & Verdict \\
\midrule
Public-node fraction & 0.10$\to$1.00 & standalone keeps 0.10$\to$1.00 of nodes & more public $\Rightarrow$ less persistent deletion room \\
Direct-call frequency $N$ & 1$\to$50 calls & persistent load $<$ transient load at all $N$ & winner: persistent-standalone \\
Cross-file redundancy & 0.35$\to$0.41 dup. & persistent deploy saving up to 28.1\% & more redundancy $\Rightarrow$ larger permanent saving \\
Bundle size (tokens) & 548$\to$879 & cold view build 170$\to$228 ms & build cost is per-run for transient, one-time for persistent \\
Update frequency $U$ & 0$\to$1 edits/run & repack 273 ms vs rebuild 183 ms at $U{=}1$ & transient wins once $U>0.1$ \\
Workload drift $d$ & 0$\to$0.8 & persistent saving 12.9\%$\to$2.6\% vs transient 4.7\% (flat) & transient wins for $d\geq0.8$ \\
\bottomrule
\end{tabular}}
\end{table*}

\begin{figure*}[t]
\centering
\includegraphics[width=0.9\textwidth]{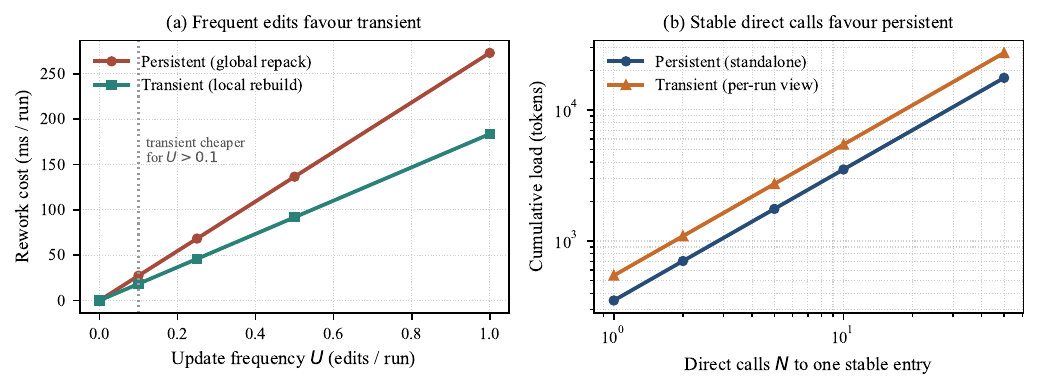}
\caption{\textbf{Which lifecycle is cheaper depends on the workload}; both panels are drawn from the deterministic sweeps of Table~\ref{tab:lifecycle-sensitivity}. \textbf{(a)} As the edit rate $U$ rises, a persistent global repack (which re-audits the whole bundle) costs more rework per run than a transient local rebuild (which touches only the affected view); they tie for a stable bundle ($U{=}0$), and transient is cheaper for $U>0.1$. \textbf{(b)} For repeated direct calls to one stable public entry, a standalone-preserving persistent bundle pays its compression once and stays below a per-run transient view at every call count $N$ (log--log axes). Read together: a stable, heavily called bundle wants persistent; a frequently edited one wants transient.}
\label{fig:lifecycle}
\end{figure*}

\subsubsection{What a transient view costs to build and cache}
Transient compression has overhead (Table~\ref{tab:lifecycle-transient}). The root-heavy entry saves \win{24.8\%}; leaf references save under 2\% because their closures share little content. Cold builds take 150--245\,ms and warm reads under 2\,ms, amortizing to 15--25\,ms per call at a 0.9 hit rate. The non-improving \texttt{geometry.md} entry uses its raw closure, and the canonical bundle remains byte-identical. A view helps only when its token saving exceeds build overhead.

\begin{table*}[t]
\centering
\caption{\textbf{Transient-view cost.} Closure and View are tokens before and after compression. Cold, Warm, and Amort. report build latency; Built marks a smaller view, and Canon. confirms that the canonical bundle is unchanged.}
\label{tab:lifecycle-transient}
\resizebox{\textwidth}{!}{%
\begin{tabular}{@{}lccccccccc@{}}
\toprule
Entry & Kind & Closure & View & Run saving~$\uparrow$ & Cold ms~$\downarrow$ & Warm ms~$\downarrow$ & Amort.\ ms~$\downarrow$ & Built & Canon. \\
\midrule
\texttt{SKILL.md} & publ & 493 & 371 & \win{24.8\%} & 218 & 0.34 & 22.1 & \cmark & \cmark \\
\texttt{sub/SKILL.md} & publ & 548 & 545 & 0.5\% & 173 & 1.53 & 18.7 & \cmark & \cmark \\
\texttt{references/algebra.md} & cond & 532 & 529 & 0.6\% & 155 & 0.23 & 15.7 & \cmark & \cmark \\
\texttt{references/edge\_cases.md} & cond & 513 & 505 & 1.6\% & 153 & 0.29 & 15.5 & \cmark & \cmark \\
\texttt{references/geometry.md} & cond & 489 & 489 & 0.0\% & 244 & 0.54 & 24.9 & \xmark & \cmark \\
\texttt{references/number\_theory.md} & cond & 483 & 478 & 1.0\% & 157 & 0.56 & 16.2 & \cmark & \cmark \\
\bottomrule
\end{tabular}}
\end{table*}

\subsubsection{Which safeguards are load-bearing}
Table~\ref{tab:lifecycle-ablation} removes one safeguard at a time. Without entry labels, public independence falls to $0.500$; without the multi-entry audit, a public entry becomes undiscoverable. Ignoring dependency closure drops 10 of 11 references, and missing host context reduces conditional independence to $0.887$. Removing the view cache raises per-call build from 18.7 to 173.4\,ms; removing global repacking moves the disk ratio from $0.845$ to $0.959$. Under workload drift, persistent saving decays from 12.9\% to 2.6\%, while the transient view remains at 4.7\%. Each safeguard therefore protects either a saving or a public call.

\begin{table*}[t]
\centering
\caption{\textbf{Lifecycle ablations.} Each row removes one safeguard and reports the protected metric and resulting failure.}
\label{tab:lifecycle-ablation}
\resizebox{\textwidth}{!}{%
\begin{tabular}{@{}lp{0.30\textwidth}ccp{0.30\textwidth}@{}}
\toprule
Removed safeguard & Metric & Intact & Ablated & What breaks \\
\midrule
$-$ Entry-type labels & mean public independence (math/structured) & 1.000 & 0.500 & public entries lose content the root happens to cover \\
$-$ Multi-entry audit & all public entries discoverable (0/1) & 1.000 & 0.000 & a public sub-entry is renamed and no longer callable \\
$-$ Dependency closure & dependency files reachable from the entry in its view & 11 & 1 & the view omits every reference the entry routes to \\
$-$ Host-context supplement & mean conditional-entry independence (math/structured) & 1.000 & 0.887 & content assumed covered by the host is deleted, then truly lost \\
$-$ View cache & mean build ms/call over 20 calls & 18.7 & 173.4 & every call rebuilds the same view from scratch \\
$-$ Global repack & deploy token ratio, lower is better (math/structured) & 0.845 & 0.959 & cross-file duplication is never removed; only per-file gains remain \\
\bottomrule
\end{tabular}}
\end{table*}

\begin{takeawaybox}
\method is the only tested approach that compresses a multi-entry bundle without breaking routing or hiding a public entry. Audited persistent compression gives the smallest bundle and lowest steady-state run cost; transient compression preserves the canonical bundle and better tolerates drift, but pays per-run build cost. The lifecycles are complementary, and each safeguard protects either fidelity or efficiency.
\end{takeawaybox}

\subsection{An Example in Industrial Compression Deployment}
\label{sec:industrial}

The preceding benchmarks use an instrumented wrapper to record which files the agent loads. We further evaluate \method inside the production execution environment of a content-moderation service. Its skill bundle contains a root document of approximately \textbf{20{,}000 tokens}, three on-demand reference documents covering prior cases and moderation standards, and a locked risk map. The root is injected when the skill is activated, while reference files are loaded only when needed. Each audit involves approximately \textbf{13 agent reasoning rounds} on average.

The official evaluation set contains \textbf{100 real moderation tasks}, evenly divided between violating and normal cases. Task content is retrieved at execution time and is not available locally, so both decision quality and runtime cost must be measured through the production harness. We report the binary moderation verdict using accuracy and false positives. The bundle is written in Chinese, which also exposes a practical limitation of language-specific protection rules: an English-only deterministic extractor identifies only \caveat{7 of 264} instruction units as required, leaving most Chinese obligations unprotected during compression.

\textbf{Aggressive compression removes essential exemption knowledge.}
Without protection classes, deterministic and model-assisted compression reduce the deployed bundle by \hl{71.4\%} and \hl{75.8\%}, respectively. However, accuracy falls from \textbf{88.00\%} for the paired uncompressed bundle to \caveat{70.00\%} and \caveat{62.00\%}. The errors are strongly asymmetric: false positives increase from 10 to 29--35, while false negatives remain between 1 and 3. An audit of the removed content shows that most losses occur in \emph{exemption rules} specifying when content should not be flagged, including authorized-seller cases, whitelisted exceptions, and evidence thresholds. As fewer of these rules are retained, the false-positive rate increases.

Unprotected compression also damages the presentation of the output interface. Two worked JSON examples and the label whitelist are separated into synthetic sections approximately sixty lines apart, while renderer scaffolding remains in the root. Although the individual fragments still exist, neither the downstream parser nor the executing model receives the interface as one coherent contract.

\textbf{Witnessed compression protects semantic and interface-critical content.}
We recompress the identical bundle while introducing the safeguards of Section~\ref{sec:witnesses} incrementally. Exemption-bearing units are locked as class C1; obligation-bearing units are restored to the required class C2, so they can be modified only through witnessed transformations; and the output interface is re-emitted as one contiguous, verbatim C0 span. A removal is accepted only when supported by a logged witness: $W_1$ for literal containment, $W_2$ for coverage-preserving transformations, or $W_3$ for entailment within evidence segments that contain no boundary marker.

Because the production harness may vary between sessions, each compressed configuration is evaluated against an uncompressed bundle in the \emph{same session}. The first uncompressed row in Table~\ref{tab:industrial} therefore reports 88.00\%, while the two sessions used by the final configuration have uncompressed baselines of 92.00\% and 91.00\%. For reference, the standalone official evaluation of the same uncompressed bundle reports 90.91\%. Table~\ref{tab:industrial} shows how compression changes as the protection mechanisms are introduced.

\begin{table}[!t]
\centering
\caption{Production evaluation under the service's native execution harness. Deployment saving is measured over the shipped bundle. Per-run saving is task-paired within the same session over 100 tasks; the final row pools two sessions ($n{=}200$). Accuracy and false positives are measured on the corresponding full evaluation set. The relevant uncompressed session baselines are 88.00\%, 92.00\%, and 91.00\%.}
\label{tab:industrial}
\resizebox{\columnwidth}{!}{%
\begin{tabular}{lccccc}
\toprule
Configuration &
Deployment~$\uparrow$ &
Measured per run~$\uparrow$ &
Accuracy~$\uparrow$ &
FP~$\downarrow$ &
Interface contract~$\uparrow$ \\
\midrule
Uncompressed
& +0.0\% & +0.0\% & 88.00\% & 10 & contiguous \\

\rowcolor{SkillZipPaleGrey}
\color{LossyGrey}
Pro, no protection classes\,\xmark
& +71.4\% & -- & \gone{70.00\% / 62.00\%}
& \gone{29 / 35} & contiguous \\

Pro $+$ C1/C2 ($W_1,W_2$ only)
& +13.8\% & +7.1\% & 89.00\% & 9
& \caveat{fragmented} \\

Pro $+$ $W_3$ entailment witness
& +32.7\% & +6.8\% & 88.00\% & 10
& \caveat{fragmented} \\

Pro $+$ $W_3$ $+$ C0 contract restore
& +32.2\% & +6.1\% & 89.00\% & 9
& \win{contiguous} \\

\rowcolor{SkillZipPaleBlue}
Pro $+$ $W_3$ inside the root (v3)
& \hl{+38.1\%}
& \hl{+10.4\%} ($n{=}200$)
& \win{89.00\%}
& \win{9}
& \win{contiguous} \\
\bottomrule
\end{tabular}}
\end{table}

The incremental comparison supports four findings.

\textbf{First, stronger witnesses expand the amount of content that can be removed safely.}
With C1/C2 protection and only $W_1$ and $W_2$, \method reduces the deployed bundle by \emph{13.8\%} and achieves \win{89.00\%} accuracy, one task above the paired uncompressed result, with one fewer false positive. Further candidates cannot pass these deterministic checks because their redundancy is semantic rather than literal. Adding the restricted $W_3$ entailment witness increases deployment saving to \hl{32.7\%}. The resulting accuracy and confusion matrix are \textbf{identical} to those of the paired uncompressed bundle. Each of the 106 removed reference segments has a logged entailment verdict, and none changes a decision on the 100-task evaluation. The increase in saving therefore comes from stronger evidence for removal, not from relaxing the acceptance criteria.

\textbf{Second, interface integrity must be evaluated separately from task accuracy.}
The configuration using $W_3$ alone leaves the output contract \caveat{fragmented} because the fixed-template renderer distributes its components across separate sections. The C0 restoration pass reconstructs a single contiguous span containing both worked examples and the label whitelist, while removing residual scaffolding. This repair reduces deployment saving only slightly, from 32.7\% to 32.2\%. Accuracy changes from 88.00\% to 89.00\%, which is insufficient to attribute an accuracy benefit to contiguity alone. The relevant result is instead that the published bundle once again exposes the complete contract in the form expected by its consumer, and that this property is now checked explicitly.

\textbf{Third, compressing the repeatedly loaded root produces the largest runtime benefit.}
Because the root is processed throughout approximately 13 reasoning rounds, removing one root token saves roughly five to six times as many measured runtime tokens as removing one token from an on-demand reference. We therefore allow restricted $W_3$ witnesses inside the root, while continuing to protect frontmatter, interface contracts, and boundary-bearing segments. Two candidate deletions are also prohibited from serving as witnesses for each other. This configuration removes \hl{38.1\%} of the deployed bundle and reduces pooled, task-paired runtime tokens by \hl{10.4\%} over $n{=}200$ audits. Both compressed runs achieve \win{89.00\%} accuracy, compared with same-session uncompressed baselines of 92.00\% and 91.00\%.

\textbf{Fourth, static cost estimates do not fully capture changes in agent execution.}
The loading-weight model of Section~\ref{sec:problem} predicts approximately \emph{3\%} per-run saving, whereas paired production measurements range from \hl{6\%} to \hl{11\%}. A shorter root changes not only the number of loaded tokens but also the agent's reading and repeated-reasoning behavior, neither of which is represented by the static model. A 30-task pilot varies from 6\% to 14\% across days, while two 100-task paired runs produce 9.5\% and 11.3\%, yielding the pooled result of \hl{10.4\%}. The static model is conservative in this deployment, but runtime claims should be based on paired measurements collected within the same execution session.

\begin{takeawaybox}
On the production moderation skill, witnessed compression removes \win{38.1\%} of the deployed bundle and \win{10.4\%} of pooled, task-paired runtime tokens while keeping accuracy within the observed variation of the uncompressed system. In contrast, the unprotected 71.4--75.8\% configurations lose \caveat{18--26 accuracy points}, primarily through additional false positives.
\end{takeawaybox}

\begin{insightbox}{5}
\textbf{The safe compression limit is determined by the evidence available for each removal.}
Literal and coverage witnesses provide only \emph{13.8\%} deployment saving, while restricted entailment witnesses raise it to \win{32.7\%}; extending the same witnessed reasoning to the repeatedly loaded root raises the final saving to \win{38.1\%}. The benefit is also amplified at execution time: although the static loading model predicts only \emph{3\%} per-run saving, the production engine measures \hl{6--11\%} because shorter root context reduces work across repeated reasoning rounds. Industrial compression should therefore be both \emph{proof-bound} and \emph{execution-aware}: remove content only when its redundancy is witnessed, and prioritize the context that the agent repeatedly processes.
\end{insightbox}

\section{Conclusion}
\label{sec:conclusion}

\method compresses the progressively loaded skill bundles used by production agents, including instructions, references, subskills, code, data, and assets. Host entailment, activation-scoped sharing, and conditional capsules remove cross-file redundancy without increasing activation or path cost. Before publication, a transactional disk audit verifies routing, typed contracts, scope, interface integrity, and locked content.

The compiler supports two independent choices. \textbf{One-Shot} compression rebuilds the complete resource graph, whereas \textbf{Continual} compression processes only the closure affected by each evolution patch. \textbf{Persistent} compression rewrites the canonical bundle, while \textbf{Transient} compression preserves it and constructs a per-run execution view. These modes support different update frequencies and requirements for the independent usability of references and subskills, while remaining compatible with Phase-A deployment.

Across three benchmarks and a real industrial deployment, \method reduces all four cost layers while preserving every routing pair and introducing no reference errors. Triggered Continual repacking achieves a result comparable to One-Shot rebuilding with approximately \win{one third} as many model calls per patch. In production, verification-gated compression removes \win{38.1\%} of the deployed bundle and \win{10.4\%} of measured per-run tokens without a measurable loss in decision quality. These results show that effective skill compression must preserve not only content, but also the loading structure through which agents access it.

\FloatBarrier
\bibliographystyle{IEEEtran}
\bibliography{references}

\clearpage
\onecolumn
\setlength{\parfillskip}{0pt plus 0.76\textwidth}
\appendices
\section{Implementation Details}
\label{app:implementation}

\begin{roadmapbox}
This appendix fixes the implementation contract needed to reproduce Phase A. It specifies bundle discovery, reference resolution, candidate layout, the cost ledger, auditing, atomic publication, and Zip-on-Write invalidation. Appendix~\ref{app:schema} gives machine-readable schemas and prompt contracts; Appendix~\ref{app:protocols} gives the run matrix; Appendix~\ref{app:analysis} expands theory, limitations, and failure analysis.
\end{roadmapbox}

\subsection{Reference Directory Layout}

The implementation treats the authoritative source directory as immutable during a transaction and creates all intermediate and persistent compiler state outside the published runtime bundle. A successful publication uses the following split layout:

\begin{lstlisting}[style=compactjson]
<published-bundle>/
  SKILL.md                   # unchanged entry path
  references/...
  subskills/.../SUBSKILL.md
  capsules/<guard-slug>.md   # generated on-demand text
  .skillzip_shared/<hash>.md # generated scoped modules
  scripts/...                # byte-identical
  data/...                   # byte-identical
  assets/...                 # byte-identical

<compiler-cache>/<bundle-id>/
  state.json                 # graph, indices, workload ledger
  manifest.json              # provenance and audit results
  contracts/<digest>.json    # reusable typed contracts
  authored.snapshot          # optional transactional source snapshot
\end{lstlisting}

Generated runtime directories are namespaced to avoid collisions. If either namespace already exists in the source, the tool selects a digest-suffixed namespace and records it in the external manifest. Nested source \texttt{SKILL.md} files are treated as subskill entries internally; the published filename is changed only when the original harness accepts \texttt{SUBSKILL.md}. The default is path preservation. Compiler state is not distributed to the agent and is therefore not charged to runtime deployment cost; a reproducibility package that includes it reports payload bytes and total archive bytes separately.

\subsection{Discovery and Media Classification}

Directory traversal is deterministic: paths are normalized to UTF-8 NFC, sorted bytewise, and visited without following directory symlinks. Content signatures and extensions classify regular files. Declared text types support contract extraction; code, notebooks, structured data, images, archives, and unknown binaries remain locked.

Every node stores:
\begin{itemize}
  \item canonical path relative to the bundle root;
  \item SHA-256 of source bytes and media classification;
  \item token and byte counts;
  \item loading class and inbound/outbound edges;
  \item whether rewrite, relocation, or only verbatim copy is permitted.
\end{itemize}

The default text-size guard is 1 MiB per file. Larger text remains locked unless the user explicitly raises the limit. Archive members are never expanded implicitly.

\subsection{Reference Resolution}

The Phase-A resolver recognizes only syntax that the host agent can already follow:
\begin{enumerate}
  \item Markdown links and images with relative local targets;
  \item explicit code-formatted paths ending in a known extension;
  \item front-matter fields declared by the selected skill format;
  \item imperative loading clauses such as ``read \texttt{references/csv.md}.''
\end{enumerate}
Fragments and query strings are separated before filesystem resolution and restored in the emitted link. Percent decoding is performed once; double decoding is forbidden. The canonical target must have the canonical root as a path-component prefix, not merely a string prefix.

The nearest enclosing heading and conditional clause define the initial guard. If guard extraction is ambiguous, the edge is marked \texttt{unknown}; it may be preserved but cannot justify capsule creation or path-specific sharing. External \texttt{http}, \texttt{https}, and \texttt{mailto} links are preserved as opaque strings and excluded from local closure checks.

\subsection{Contract Extraction and Normalization}

Files are chunked only at heading boundaries, with a 256-token overlap carrying the parent heading and active guard. Each chunk is extracted once. Units are merged by source span and normalized key; disagreements retain the stricter wording as a locked residual. The implementation never resolves a disagreement by majority vote.

Normalization is type specific. Tool calls normalize tool name, required arguments, and order. Output obligations normalize field name, type, and cardinality. Prohibitions preserve polarity. Workflow units form a directed graph whose edges encode \texttt{before}, \texttt{after}, \texttt{retry}, or \texttt{fallback}. Examples are tagged as non-normative unless the source explicitly declares them required.

\subsection{Protection Classes, Witnesses, and Non-Latin Bundles}

The production configuration in Section~\ref{sec:industrial} adds three implementation requirements.

\textbf{Script-aware lexical signals.} Modality and requirement detection are lexical: obligation, prohibition, and recommendation markers drive which units are required. The reference implementation ships with English marker sets, which on a Chinese bundle demotes almost every obligation to optional (7 of 264 units required on the production skill). The fix is script-aware signal sets---obligation and prohibition markers per writing system---applied in the same deterministic extractor, so required-unit detection works before any witness is consulted. The English path is unchanged byte for byte.

\textbf{Boundary classes and the exemption lock.} Units whose payload carries an exemption or negative-verdict marker (authorized cases, whitelisted exceptions, ``judge normal'' clauses) are classified as boundary units. Boundary units are required and additionally locked against wording compression and against $W_3$ entailment removal: they may only be moved verbatim. Markers are a configurable lexicon per language; the audit counts them so a release reports boundary retention explicitly.

\textbf{Interface-contract restoration.} After materialization, a post-pass locates each source section matching the interface-contract pattern (output-format headings with fenced schema examples and whitelist constraints), strips any of its normalized lines that the fixed-template render scattered elsewhere in the emitted root, removes render-scaffolding sentences, and re-inserts the whole source span verbatim at one site. The audit then verifies contiguity: the section must appear once, whole, with no fragments outside it.

\textbf{Entailment witness protocol ($W_3$).} The checker prompt presents the emitted root document $S$ and one candidate reference segment $P$, and asks whether the decision-relevant content of $P$---conditions, verdicts, thresholds, whitelists, exemptions---is fully expressed in $S$, explicitly instructing that concrete case facts need not reappear but the rule they demonstrate must. The checker runs at temperature $0$ with a fixed prompt; the first response line must be the verdict and the second a one-sentence basis. A verdict is accepted only for evidence-class segments containing no boundary marker, and the verdict, basis, segment coordinates, and prompt digest are written to the witness log, making every $W_3$ removal individually auditable after publication.

\textbf{Intra-root witnesses and their safeguards.} $W_3$ may also be applied inside the root document, where each candidate is checked against the root \emph{minus itself}. Three safeguards apply. (i) The YAML frontmatter and every interface-contract span are exempt from $W_3$ regardless of verdict: they carry no decision knowledge a checker can weigh, yet they are the loading and output contract. (ii) Boundary segments remain exempt. (iii) A mutual-witness veto: after verdicts are collected, any two approved segments whose normalized bigram sets overlap above $0.6$ are both retained, since each may have received its verdict only because the other was still present in the base. Because the root is injected into every reasoning round, intra-root removals carry the largest per-run multiplier, and the veto plus the exemptions are what keep that lever safe.

\subsection{Candidate Enumeration}

File-level candidates are primitive statements, shared rules, named procedures, guarded exceptions, and locked residual spans. Bundle-level candidates are exact host witnesses, shared modules, and capsules.

Cross-file sharing uses a two-stage test. A hash over canonical typed units finds exact candidates; a deterministic structural comparator then verifies identical type, payload, guard parameters, and normative strength. Parameterized sharing is allowed only when differing values correspond to explicit source variables and the resulting call sites preserve them. Pure embedding similarity never creates an automatic shared module.

For each support set, the algorithm considers the lowest common activation scopes in the resource graph rather than only the directory ancestor. A scope candidate is discarded if it would be loaded by a path that previously loaded none of the occurrences, unless the added path cost is compensated under Eq.~\eqref{eq:objective} and the rule is proven universally applicable. The production default disables this exception and requires no scope expansion.

Capsule candidates require all of the following: an explicit guard, a body above the minimum length, no unguarded inbound dependency from adjacent prose, and a dispatcher whose mandatory loading instruction fits within the configured budget. Headings such as ``Background'' or ``Notes'' do not constitute guards.

\subsection{Selection and Cost Ledger}

Candidate deltas use tokenized emitted text, including links, dispatcher verbs, generated headings, and manifest-excluded runtime content. Selection starts with the best coverage-preserving improvement, then tests removals and one-for-one or one-for-two swaps until $J$ cannot improve by one token.

The ledger stores both estimated and materialized values:
\begin{lstlisting}[style=compactjson]
{
  "tokenizer": {"name": "...", "revision": "..."},
  "lambda_deployment": 0.05,
  "path_weights": {"source": "trace|guards", "digest": "..."},
  "source": {"catalog": 0, "activation": 0,
             "deployment": 0, "mean_path": 0, "p95_path": 0},
  "candidate": {"catalog": 0, "activation": 0,
                "deployment": 0, "mean_path": 0, "p95_path": 0},
  "objective_delta": 0
}
\end{lstlisting}
Estimated and materialized costs must agree exactly for bytes and within tokenizer determinism for tokens. A discrepancy rejects the candidate.

\subsection{Disk-Level Audit}

The audit runs in a new process with only the source path, candidate path, and manifest as inputs. This prevents it from trusting optimizer objects accidentally. It rebuilds both graphs, rehashes every node, and verifies:
\begin{enumerate}
  \item source and candidate roots are distinct and canonical;
  \item all internal candidate targets exist, remain inside the root, and preserve fragments;
  \item every source contract unit maps to a compatible candidate unit or exact host witness;
  \item every generated shared module or capsule is reachable through a mandatory dispatcher under a compatible guard;
  \item locked-file path and SHA-256 pairs match, except for explicitly user-approved path migrations;
  \item no source file has been omitted merely because it was unreferenced;
  \item every interface contract appears as one contiguous span equal to its source section, with no scattered fragments;
  \item materialized costs do not exceed the source objective.
\end{enumerate}

Strict mode treats any warning as failure. Diagnostic mode emits the candidate for inspection but never publishes it as a successful compression.

\subsection{Atomic Publication and Recovery}

The tool creates a temporary sibling directory so that final rename remains on one filesystem. It fsyncs files, the manifest, and containing directories before publication. If the destination exists, the default command fails; explicit \texttt{--replace} first renames the old output to a timestamped backup. The source is never overwritten.

On one-shot extraction, selection, materialization, or audit failure, the official output is a verbatim copy of the source with an external failure manifest. On a continual optimization failure after a patch has been applied successfully, the official output is the verbatim patched bundle. Only a failure to apply or validate the patch transaction itself leaves the previous publication active. A temporary directory may be retained for debugging only under an explicit flag. This behavior makes failure visible while preserving the current authored semantics.

\subsection{Dual-Mode Driver and Bundle-Aware Zip-on-Write}

The command exposes \texttt{--mode one-shot} and \texttt{--mode continual}. One-shot requires only the source and entry path; it rebuilds all compiler state. Continual additionally requires the previous bundle identifier, state digest, and an incoming patch. A missing, corrupt, or source-mismatched state never causes an unsafe partial update: the driver applies the patch verbatim and switches to one-shot compression of that current authored state.

An incoming patch identifies added, modified, moved, and deleted paths. Before optimization, the tool creates a raw patched snapshot and validates patch preconditions and expected hashes. The invalidation closure contains changed nodes, their reference ancestors, newly or formerly referenced descendants, shared modules whose support sets changed, capsules whose guard spans changed, and environment witnesses whose digest changed. Extraction and candidate generation rerun only on this closure. Global reference, coverage, and hash checks still scan the full candidate, while unchanged file contracts are reused by source digest.

Each patch unit is classified as \textsc{Absorb}, \textsc{Refine}, \textsc{Extend}, or \textsc{Refactor}. The first three update semantic content. \textsc{Refactor} may also change placement when a support set, guard, or path weight crosses the sharing or capsule threshold. The driver estimates recoverable saving from stale candidates and invokes one-shot repacking when any configured trigger is reached: $\theta_{\mathrm{repack}}$ saving, $\rho$ growth, $\delta_W$ workload drift, environment-digest change, or $K$ patches. Trigger values and causes are logged in the state transition.

Deletion is especially conservative. A shared module is deleted only after all inbound dispatchers are removed and the recomputed coverage graph shows no remaining source unit depends on it. If compression of a valid raw patch fails, the raw patched snapshot is atomically published and a clean state is rebuilt from it. Thus a patch can reduce compression temporarily but cannot disappear because the optimizer rejected its rewrite.



\section{Schemas and Model Contracts}
\label{app:schema}

This appendix records the structured interfaces used by the extractor and rewriter. Prompts are templates; implementations should pin the system message, model snapshot, JSON validator, and retry policy.

\subsection{Resource-Graph Record}

\begin{lstlisting}[style=compactjson]
{
  "bundle_digest": "sha256:...",
  "entry": "SKILL.md",
  "nodes": [{
    "id": "n17",
    "path": "references/csv.md",
    "sha256": "...",
    "media_type": "text/markdown",
    "kind": "reference",
    "locked": false,
    "token_count": 412
  }],
  "edges": [{
    "source": "n1", "target": "n17",
    "source_span": [88, 126],
    "syntax": "markdown_link",
    "guard": {"text": "when exporting CSV", "status": "explicit"}
  }],
  "external_edges": [],
  "unresolved": []
}
\end{lstlisting}

Line/byte offsets are measured against immutable source bytes. A graph containing an unresolved local reference cannot enter strict compression.

\subsection{Typed Contract Record}

\begin{lstlisting}[style=compactjson]
{
  "resource_id": "n17",
  "units": [{
    "unit_id": "n17:u4",
    "type": "output_obligation",
    "key": "csv.encoding",
    "payload": {"value": "UTF-8", "strength": "must"},
    "scope": {"resource": "n17", "guard": "export CSV"},
    "source_spans": [[241, 274]],
    "confidence": "high",
    "locked": false
  }],
  "relations": [{
    "kind": "before",
    "source_unit": "n17:u2",
    "target_unit": "n17:u4",
    "source_spans": [[178, 274]]
  }]
}
\end{lstlisting}

Allowed unit types are \texttt{interface}, \texttt{workflow}, \texttt{tool}, \texttt{rule}, \texttt{prohibition}, \texttt{output\_obligation}, \texttt{evidence}, and \texttt{locked\_residual}. The validator rejects unknown types and units without source spans.

\subsection{Environment Contract and Witness}

\begin{lstlisting}[style=compactjson]
{
  "environment_digest": "sha256:...",
  "guarantees": [{
    "type": "output_obligation",
    "key": "csv.encoding",
    "value": "UTF-8",
    "scope": "all_spreadsheet_exports",
    "enforced_by": "serializer/v3",
    "evidence_digest": "sha256:..."
  }]
}
\end{lstlisting}

A removal witness stores the source unit, guarantee index, exact comparison result, and environment digest. Natural-language descriptions of the host are not accepted as contracts.

\subsection{Extraction Prompt Contract}

\begin{promptbox}[Contract extraction]
You are extracting an auditable behavioral contract from ONE resource in an agent-skill bundle.

Inputs: canonical resource path; ancestor headings; explicit loading guard; immutable source text with byte offsets; resource-graph neighbors.

Return JSON matching the supplied schema. Extract every applicability condition, ordered workflow step, tool/resource requirement, normative rule, prohibition, output obligation, verification/evidence requirement, and ambiguous residual. Attach exact source spans. Preserve polarity, strength, and guard. Mark uncertain or non-decomposable text as locked\_residual.

Do not rewrite the resource. Do not infer requirements from world knowledge. Do not treat an example as normative unless the source explicitly does. Do not invent references or guards. Output JSON only.
\end{promptbox}

\subsection{Candidate Rewrite Prompt Contract}

\begin{promptbox}[File rewrite]
Rewrite the supplied Markdown resource using ONLY the selected representation plan and mapped source units.

Hard requirements:
(1) preserve every selected typed unit at the specified scope and normative strength;
(2) preserve every locked residual verbatim;
(3) emit each provided relative path exactly;
(4) for a capsule/shared module, keep the supplied explicit guard and mandatory read instruction;
(5) do not add tools, claims, examples, branches, or host assumptions;
(6) return one Markdown file with no commentary.

If the plan is inconsistent or omits a mapped unit, return PLAN\_REJECTED with the unit IDs.
\end{promptbox}

The implementation does not rely on the rewriter to calculate coverage. It re-extracts and audits the emitted text.

\subsection{Audit Prompt Contract}

Most audit checks are deterministic. A model is used only for semantic coverage pairs not discharged by exact normalization.

\begin{promptbox}[Semantic coverage adjudication]
Given one source unit and one candidate unit, decide whether the candidate entails the full source requirement under the same or narrower compatible scope.

Return: COVERED, NOT\_COVERED, or UNCERTAIN; a type match; a polarity/strength match; a scope match; and the minimal text spans supporting the decision.

Rules: examples do not cover obligations; recommendations do not cover MUST; broader applicability cannot be assumed from a narrower branch; external knowledge and model plausibility are forbidden. UNCERTAIN is treated as NOT\_COVERED by strict mode. Output JSON only.
\end{promptbox}

\subsection{Transformation Manifest}

\begin{lstlisting}[style=compactjson]
{
  "format_version": "skillzip-pro/1",
  "phase": "A",
  "source_digest": "sha256:...",
  "output_digest": "sha256:...",
  "models": {"extractor": "...", "rewriter": "...", "auditor": "..."},
  "graph": {"nodes": 0, "edges": 0, "unresolved": 0},
  "transformations": [
    {"kind": "host_entailment", "units": [], "witness": "..."},
    {"kind": "scoped_share", "units": [], "scope": [], "path": "..."},
    {"kind": "capsule", "units": [], "guard": "...", "path": "..."},
    {"kind": "file_cover", "resource": "...", "units": []}
  ],
  "cost_ledger": {},
  "audit": {"passed": true, "checks": [], "fallback": false}
}
\end{lstlisting}

The manifest is not required at agent runtime. It exists for provenance, regression testing, and exact reconstruction of experimental measurements.

\section{Detailed Experimental Protocols}
\label{app:protocols}

\subsection{Pre-registration Sequence}

The evaluation follows a fixed order to prevent test leakage:
\begin{enumerate}
  \item freeze benchmark versions, evolution/validation/test IDs, and task-template groups;
  \item construct native and structured-growth source bundles using only evolution tasks;
  \item freeze the bundle builder, path annotations, environment contracts, and source digests;
  \item tune only declared thresholds on validation tasks;
  \item freeze compressor code, prompts, model snapshots, tokenizers, and non-inferiority margin;
  \item run every compressor on the same source digests;
  \item audit outputs before any held-out execution;
  \item execute all valid or fallback outputs on paired held-out tasks;
  \item generate tables directly from immutable raw records.
\end{enumerate}
Any post-freeze repair creates a new experiment revision and reruns all affected baselines.

\subsection{Bundle Construction Checks}

The structured-growth builder moves contiguous source spans into references or subskills and replaces them with explicit guarded links. Before compression, three checks establish that this transformation is not itself responsible for performance changes: (i) concatenating files in provenance order recovers all normative source spans; (ii) every moved span has exactly one reachable dispatcher; and (iii) the evolved flat skill and the structured bundle are compared on validation tasks under the same executor. A task-score difference above one percentage point triggers manual review and rebuilding.

Each benchmark includes at least four branch types: a common path, a rare format/tool path, a verification path, and a recovery path. Path-frequency scenarios include observed validation traces, uniform leaf paths, and a skewed stress distribution with 80\% mass on the common path. This separates gains from bundle structure from gains tied to one traffic assumption.

\subsection{Gold Resource Sets}

For each held-out task, annotators identify the minimal set of source resources required to satisfy the explicit task and benchmark rubric. They may inspect the task, source bundle, tool schema, and official expected behavior, but not any compressed output. A second annotator independently labels a stratified 20\% sample covering all branches. Disagreements are adjudicated before execution; Cohen's $\kappa$ and raw agreement are reported.

A resource is required if omitting it removes a unique normative unit needed by the task. Redundant references are marked acceptable alternatives rather than jointly required. Scripts invoked by a required workflow count as required even if their source code is not injected into the model context.

\subsection{Run Matrix}

The full factorial matrix is summarized in Table~\ref{tab:runmatrix}. Budget-matched and native-cost comparisons are separate run groups. The no-skill condition is executed but not compressed.

\begin{table}[h]
\centering
\caption{Minimum run matrix before stochastic seeds.}
\label{tab:runmatrix}
\begin{tabular}{lr}
\toprule
Dimension & Levels \\
\midrule
Benchmarks & 3 \\
Bundle tracks & 2 \\
Compression methods & 7 \\
Compressor models & 3 \\
Executor models & 3 \\
Traffic scenarios (cost only) & 3 \\
Primary seeds & 3 if stochastic \\
\bottomrule
\end{tabular}
\end{table}

For the main table, compressor model and traffic estimator are fixed in advance, reducing the execution grid to all methods, benchmarks, tracks, executors, tasks, and seeds. The full compressor--executor cross product is used only for RQ4.

\subsection{Harness Instrumentation}

Instrumentation records timestamp, task ID, selected skill, read path, read result, byte count, model-token count, caller action, and current branch label. It does not expose gold paths to the agent. Catalog and root injection are recorded as synthetic events emitted by the existing harness; auxiliary loads are recorded at the ordinary file-reading boundary.

Two validation tests are mandatory. The \emph{visibility test} verifies that the agent cannot read the trace buffer. The \emph{equivalence test} replays a fixed set of actions with instrumentation on and off and compares all task-visible observations byte-for-byte. Hashes of the harness executable and configuration accompany each trace.

\subsection{Task Execution and Failure Policy}

A compression-time exception, audit rejection, or timeout produces the original bundle as the executable fallback. Its task score and uncompressed costs remain in the method's aggregate; excluding failed cases would reward unsafe aggressiveness. An execution-time missing file, invalid path, or unavailable generated module counts as both task failure (when it prevents completion) and a loading error.

Task timeouts and transient provider failures are distinguished. A provider failure is retried under the same seed up to the predeclared limit; a reproducible model or agent failure is not retried. Every exclusion is listed by task ID and reason.

\subsection{Metric Computation}

Catalog and activation cost are deterministic functions of published files. Deployment cost includes all files distributed to the agent except the audit manifest when the production package explicitly excludes it; both inclusive and runtime-package counts are recorded. Path cost is calculated from actual load events, including duplicate loads. Mean and P95 are paired over identical tasks.

For a task $t$, required-file recall is defined as one when $R_t$ is empty. Irrelevant-load rate is zero when $L_t$ is empty. Generated shared modules inherit the union of source-unit annotations they cover, allowing them to count as a valid substitute for original resources. A capsule is relevant only when its guard is active for the task.

Quality macro-averages first within benchmark categories, then across the three benchmarks, preventing a large benchmark from dominating. Compression ratios are computed from summed token costs before averaging; we additionally report median per-bundle reduction to reveal size effects.

\subsection{Confidence Intervals and Tests}

The task is the resampling unit. Paired bootstrap samples preserve all method outcomes, path costs, and seeds for a task. The primary comparison is \method versus Evolved Bundle for quality and versus Root-only \base for mean-path reduction. Secondary comparisons cover other baselines and cost layers.

We declare quality non-inferiority when the lower bound of the paired 95\% interval exceeds $-0.01$. Only after non-inferiority is established do we test compression superiority. McNemar's test operates on paired binary outcomes; multi-level official scores use paired bootstrap or permutation tests. Holm correction is performed separately within RQ2--RQ7.

\subsection{Table Population Contract}

Every cell in Section~\ref{sec:experiments} is emitted by the table script from a
row-oriented result file; nothing is transcribed by hand. The pipeline writes one
JSONL record per (benchmark, track, method) with the following shape, and the
script renders each table body from those records:
\begin{lstlisting}[style=compactjson]
{
  "benchmark": "...", "track": "native|structured",
  "method": "pro|root_only|flat_concat|skillreducer|expert|...",
  "cost": {"catalog_tokens": 0, "activation_tokens": 0,
           "deployment_text_tokens": 0, "path_mean_tokens": 0.0,
           "path_max_tokens": 0, "path_count": 0},
  "J": 0.0,
  "reductions": {"catalog": 0.0, "activation": 0.0,
                 "deployment": 0.0, "mean_path": 0.0,
                 "max_path": 0.0, "J": 0.0},
  "loading": {"orphaned_modules": 0, "ref_errors": 0,
              "dangling": 0, "unsafe": 0, "reachable_files": 0},
  "routing": {"pairs_before": 0, "pairs_exact_kept": 0,
              "routing_fidelity": 0.0},
  "efficiency": {"time_s": 0.0, "compress_calls": 0,
                 "rollouts": 0, "peak_gb": 0.0},
  "report": {"selected_verbatim": false, "audit_ok": true,
             "promotions_count": 0, "capsules_count": 0,
             "env_drops_count": 0}
}
\end{lstlisting}
Quality and behavioral-loading records are stored separately, keyed by the same
(benchmark, track, method) triple plus an executor field, and carry
\texttt{accuracy}, \texttt{required\_recall}, \texttt{irrelevant\_load},
\texttt{dispatch}, and the per-task read traces. Continual records add the
per-patch trajectory with $J(B_t^{\mathrm{one}})$, $J(B_t^{\mathrm{cont}})$, the
regret of Eq.~\eqref{eq:continualregret}, and an explicit repack flag.

The generated artifacts are: \texttt{costs.jsonl} (method costs and loading
structure), \texttt{ablations.jsonl}, \texttt{quality.jsonl},
\texttt{transfer.jsonl}, \texttt{continual.jsonl},
\texttt{lambda\_sweep.jsonl}, and \texttt{repack\_sensitivity.jsonl}. A failed
compression remains in the denominator and is represented by the cost of its
verbatim fallback, which the record marks with \texttt{selected\_verbatim}.

\subsection{Zip-on-Write Replay}

Each evolution patch is a transaction with timestamp, added/modified/moved/deleted paths, inserted source spans, and expected source digest. All schedules see patches in identical order and operate on the same verbatim authored checkpoint: append-only, Pro One-Shot after every patch, root-only \base Zip-on-Write, Pro Continual without repacking, and Pro Continual with the combined repack policy. We checkpoint after every patch and measure changed contract units, invalidation-closure size, reused-contract fraction, rewritten bytes, model calls, update latency, four costs, objective regret, audit outcome, fallback type, and validation quality. Held-out test quality is measured only at predeclared checkpoints to avoid adaptive tuning to the test set.

The growth plot uses patch index on the horizontal axis and reports absolute tokens, not only normalized reduction. Additional panels report objective regret, cumulative compression cost, and closure fraction so that frequent small updates are not favored by hiding total overhead. Repack events and their trigger causes are overlaid. Every rejected compression is checked automatically to ensure that the published fallback digest equals the verbatim patched snapshot, not the previous version.

\subsection{Exact Setup of the Reported Runs}
\label{app:setup}

This subsection records, in full, the configuration behind every number in
Section~\ref{sec:experiments}, so that a reader can reproduce it without reading
the code.

\textbf{Bundle tracks.} Three tracks are used. (i) \emph{Plain}: the evolved class
skills assembled into a root with a routing list, one reference per class, one class
as a nested subskill, and one locked helper script. (ii) \emph{Grown}: the same
skills with evolution patches replayed into named branches, which re-append the
output rules, the verification checklist and the accumulated pitfall list into every
branch, plus one file of long guarded edge cases. A deterministic builder records
the source span and patch identifier of every moved section and never rewords text.
(iii) \emph{Real evolved libraries}: three self-evolution runs stored with the
project, each keeping all of its rounds as separate files, giving 17 files, 15
rounds, roughly 20{,}000 tokens, and 79--84\% repeated text. Main cost tables use
the grown track; Sections~\ref{sec:evolved}--\ref{sec:startearly} use the real
libraries.

\textbf{Held-out task counts.} Math 33, BFCL 21, SpreadsheetBench 48, giving 102
tasks pooled. The same tasks are used for every method within a benchmark, which is
what makes the paired test valid. SpreadsheetBench is sampled at twice the density
of the other two because its verdict is the strictest (a task passes only if every
one of its test cases reproduces the gold range), so single-task flips dominate at
small $n$; at $n=24$ the ordering between the compressed and uncompressed bundles
turned on one task, and doubling the sample removed that artefact without changing
any method's configuration. BFCL runs with its offline search tool enabled;
without the tool the skill provides no measurable benefit on that benchmark (no
skill and the uncompressed library score identically), so that configuration cannot
separate compression methods and is not used.

\textbf{Compressor configuration.} Structural extraction and structural audit
(no model calls), capsule minimum 40 tokens, shortest repeated statement worth
factoring 8 tokens, storage weight $\lambda=0.05$, navigation locked, literal
blocks locked, and pre-existing source reference defects tolerated but never
introduced. Every method receives the identical source bundle and the identical
environment contract.

\textbf{Environment contract.} Guarantees are only those the benchmark harness
itself puts in front of the agent on every task, quoted from the harness: for math,
``Solve the PROBLEM and end with the final result on its own line as
\texttt{\textbackslash boxed\{...\}}''; for BFCL, ``Decompose it into ordered hops
and resolve each one'' and ``end with a line EXACTLY in the form `ANSWER: <short
answer>'\,''; for SpreadsheetBench, ``Read the workbook at path IN and write the
modified workbook to path OUT'', ``Respond with ONE \texttt{run\_python} block
\ldots{} the variables IN and OUT are predefined'', and ``write general logic (do not
hard-code values you can compute)''. A statement the harness does not supply is
never listed as a guarantee, because removing it would take away a rule the agent
still needs.

\textbf{Knowledge kept.} A content unit is one behaviour-bearing line of the source
(a bullet or sentence of at least a few words, excluding headings and fenced code).
A unit counts as kept if its wording still appears, verbatim or lightly reworded,
in a file the agent can reach by following links from the root; matching accepts
either high local edit similarity or high local coverage of the unit's content
words inside one window, so moving a qualifier is not scored as a loss while
scattering words across unrelated files is. Units removed under the environment
contract are reported separately as witnessed removals.

\textbf{Continual replay.} Rounds are appended one at a time in evolution order.
All five schedules start from the identical state and see the identical rounds.
Repacking runs every fourth round after compression is switched on. Drift is the
relative difference in objective against rebuilding the whole library at that same
round. Model calls are counted as one per changed textual node for a local update
and one per textual node for a rebuild.

\textbf{Start-time ablation.} Compression is switched on at round
$k\in\{1,4,7,10,13\}$, and for comparison never. Everything else is held fixed. The
reported total is the sum, over rounds, of the per-run cost of the library published
at that round, which is what an agent actually pays while the library evolves.

\subsection{Compounding Analysis (Section~\ref{sec:compounding})}
\label{app:compounding}

The three findings behind the compounding effect are computed from the same
full-resolution real evolved libraries used for the continual study, with no new
model calls. For each round $k$ we build the prefix holding the first $k$ rounds,
measure its raw shipped size $U(k)$, compress it one-shot with \method{} to size
$C(k)$, and record both. The compression ratio at round $k$ is $1-C(k)/U(k)$
(Figure~\ref{fig:compounding}a). The repeated share of round $k$ is
$1-\big(C(k)-C(k-1)\big)/\big(U(k)-U(k-1)\big)$: the numerator is the genuinely new
text the compressor kept, the denominator is all text the round added, so their
complement is the fraction of the round that was already present
(Figure~\ref{fig:compounding}b). Reuse concentration counts, for each shared module
\method{} created, how many strategy files link to it, sorted descending
(Figure~\ref{fig:compounding}c). We report the two libraries whose one-shot
compression cleared the never-inflate check at every prefix; the third is excluded
because at some prefixes its candidate did not beat the source and was republished
verbatim, which would report a $0\%$ ratio that reflects the safety guard rather
than the redundancy structure under study. All numbers are averaged across the
included libraries and are monotone in the same direction for each one
individually.

\section{Additional Analysis}
\label{app:analysis}

\subsection{Why a Deployment-Only Objective Is Insufficient}

Consider a bundle with a 200-token root, two rare references of 800 tokens each, and an exact 300-token fragment shared by those references. Each rare branch is used by 5\% of tasks. Global factoring into the root reduces deployment by nearly 300 tokens (minus references), yet increases expected runtime exposure by approximately $0.9\times300=270$ tokens per task. Factoring into a shared on-demand module retains the deployment saving while charging the module only to the 10\% affected paths. This example illustrates why deduplication scope is part of the representation, not an implementation detail.

The same reasoning applies to capsules. Moving a 600-token export guide behind a 20-token dispatcher at 10\% frequency saves roughly $600-(20+60)=520$ expected path tokens before deployment weighting. At 95\% frequency, the saving falls to 10 tokens and may disappear after reference overhead. The path distribution changes the best layout even when the underlying prose does not.

\subsection{Sensitivity to Unknown Traffic}

When traces are unavailable, uniform leaf-path weights can overvalue rare branches in a deep tree or undervalue a broad common branch. \method therefore exposes three policies: \texttt{uniform-leaf}, \texttt{guard-prior}, and \texttt{minimax}. The minimax policy accepts a transformation only when it improves $J$ for every distribution in a declared uncertainty set. It is more conservative but avoids large regressions under traffic shift.

We recommend reporting both the optimization distribution and counterfactual distributions. A transformation that wins only under a narrow prior should be labeled workload-specific. Host-entailment witnesses are independent of traffic, while capsule and sharing placement are not.

\subsection{Faithfulness Is Structural, Not Universal Semantics}

Typed coverage protects explicit instructions and their scope, but natural language can carry implicature, tone, and redundancy that help a particular model. Removing redundant emphasis may affect behavior even if normative content is retained. This is why the empirical protocol measures non-inferiority across executor families and why a strict structural audit does not replace held-out evaluation in scientific validation. ``Evaluation-free'' describes the compression algorithm, not the evidence standard for publishing its effectiveness.

The audit can also share errors with extraction. Independent prompts, source-span provenance, deterministic type checks, and verbatim locking reduce correlated failure but cannot eliminate it. High-stakes skills should use human review or stronger formal specifications for critical obligations.

\end{document}